%% file: root.tex
\documentclass[a4paper,conference]{IEEEtran}
\IEEEoverridecommandlockouts
\usepackage[sort, compress]{cite}
\usepackage{textcomp}
\include{math_commands}

\usepackage[ruled, lined, linesnumbered]{algorithm2e}
\usepackage{url}
\usepackage{xcolor}
\usepackage{fancyhdr}
\usepackage[colorlinks = true,
            linkcolor = blue,
            urlcolor  = blue,
            citecolor = blue,
            anchorcolor = blue]{hyperref}

\usepackage{enumitem}
\usepackage{rotating}
\usepackage{tabulary}
\usepackage{multirow}

\usepackage{graphicx, epstopdf}
\usepackage{wrapfig}
\usepackage{caption}
\usepackage{subcaption}
\graphicspath{{./figs/}}

\newtheorem{theorem}{Theorem}
\newtheorem{definition}{Definition}
\newtheorem{lemma}{Lemma}
\newtheorem{proposition}{Proposition}
\newtheorem{corollary}{Corollary}[theorem]
\theoremstyle{remark}
\newtheorem{remark}{Remark}

\usepackage{booktabs} 
\def\BibTeX{{\rm B\kern-.05em{\sc i\kern-.025em b}\kern-.08em
    T\kern-.1667em\lower.7ex\hbox{E}\kern-.125emX}}
\begin{document}
\title{DeepNNK: Explaining deep models and their generalization using polytope interpolation
\thanks{Our work was supported by DARPA's LwLL program (FA8750-19-2-1005).}
}

\author{\IEEEauthorblockN{Sarath Shekkizhar, 
 Antonio Ortega}
\IEEEauthorblockA{\textit{Department of Electrical and Computer Engineering} \\
\textit{University of Southern California}\\
Los Angeles, CA, USA \\
shekkizh@usc.edu, aortega@usc.edu}
}
\fancyhead{}
\renewcommand{\headrulewidth}{0pt}
\lfoot{Submitted for review at ICPR 2020}
\rfoot{}
\maketitle

\begin{abstract}
Modern machine learning systems based on neural networks have shown great success in learning complex data patterns while being able to make good predictions on unseen data points. 
However, the limited interpretability of these systems hinders further progress and application to several domains in the real world. This predicament is exemplified by time consuming model selection and the difficulties faced in predictive explainability, especially in the presence of adversarial examples. 
In this paper, we take a step towards better understanding of neural networks by introducing a local polytope interpolation method. The proposed Deep Non Negative Kernel regression (NNK) interpolation framework is non parametric, theoretically simple and geometrically intuitive. We demonstrate instance based explainability for deep learning models and develop a method to identify models with good generalization properties using leave one out estimation. Finally, we draw a rationalization to adversarial and generative examples which are inevitable from an interpolation view of machine learning.
\end{abstract}


\section{Introduction}
The goal of any learning system is to identify a mapping from input data space to output classification or regression space based on a finite set of training data with a basic generalization requirement: 
Models trained to perform well on a given dataset (empirical performance) should perform well on future examples (expected performance), i.e., the gap between expected and empirical performance must be small. 

Today, deep neural networks are at the core of several recent advances in machine learning. 
An appropriate deep architecture is closely tied to the dataset on which it is trained and is selected with significant manual engineering by practitioners or by random search based on \emph{subjective heuristics} \cite{blum2015Ladder}.
Approaches based on resubstitution (training) error, which is often near zero in deep learning systems, can be misleading, while held out data (validation) metrics introduce possible selection bias and the data they use might be more valuable if it can be used to train the model\cite{anders1999model}. 
However, these methods have steadily improved state of the art metrics on several datasets even though only limited understanding of generalization is available \cite{recht2018cifar} and generally it is not known whether a  smaller model trained for fewer epochs could have achieved the same performance \cite{castelvecchi2016can}.

A model is typically said to suffer from overfitting when it performs poorly to test (validation) data while performing well on the training data. The conventional approach to avoid overfitting with error minimization is to avoid training an over-parameterized model to zero loss, for example by penalizing the training process with methods such as weight regularization or early stopping \cite{hastie2009elements, murphy2012machine}. 
This perspective has been questioned by recent research, 
which has shown that a model with a number of parameters several orders of magnitude bigger than the dataset size, and trained to zero loss, generalizes to new data as well as a constrained model \cite{zhang2016understanding}. 
Thus, while conventional wisdom about interpolating estimators \cite{hastie2009elements, murphy2012machine} is that they can achieve zero training error but generally exhibit poor generalization, 
Belkin and others \cite{belkin2018overfitting, belkin2018does} propose and theoretically study some specific interpolation-based methods, such as simplicial interpolation and kernel weighted and interpolated nearest neighbors~(wiNN), that can achieve generalization with theoretical guarantees.    
\cite{belkin2018overfitting} suggests that  neural networks perform interpolation in a transformed space and that this could help explain their generalization performance. 
Though this view has spurred renewed interest in interpolating estimators\cite{liang2018just, hastie2019surprises}, 
there have been no studies of interpolation based classifiers \textit{integrated} with a complete neural network. 
This is due in part to their complexity: working with $d$-simplices \cite{belkin2018overfitting} would be impractical if the dimension of the data space $d$ is high, as is the case for problems of interest where neural networks are used. In contrast, a simpler method such as wiNN does not have the same geometric properties as the simplex approach. 

In this paper, we propose a practical and realizable interpolation framework based on local polytopes obtained using Non Negative Kernel regression (NNK)\cite{shekkizhar2020} on neural network architectures. 
As shown in a simple setup in \Figref{fig:interpolation_difference}, a simplicial interpolation, even when feasible, constrains itself to a simplex structure (triangles in $\R^2$) around each test query, which leads to an arbitrary choice of the containing simplex when data lies on one of the simplicial faces. Thus, in the example of \Figref{fig:interpolation_difference} only one of the triangles can be used, and only two out of the 4 points in the neighborhood contribute to the interpolation. 
This situation becomes increasingly common in high dimensions, worsening interpolation complexity. 
By relaxing the simplex constraint, one can better formulate the interpolation using generalized convex polytope structures, such as those obtained using NNK, that are dependent on the sampled training data positions in the classification space. While our proposed method uses $k$ nearest neighbors (KNN) as a starting point,  
it differs from other KNN-based approaches, such as wiNN schemes \cite{devroye1998hilbert, biau2015lectures, belkin2018overfitting} and DkNN \cite{Papernot2018, wallace2018interpreting}.
In particular, these KNN based algorithms can be potentially biased if data instances have different densities in different directions in space. Instead, as shown in  \Figref{fig:KRI_plane} NNK automatically selects data points most influential to interpolation based on their relative position, i.e., only those neighboring representations that provide new (orthogonal) information for data reconstruction are selected for functional interpolation.
In summary, our proposed method combines some of the best features of existing methods, providing a geometrical interpretation and performance guarantees as the simplicial interpolation  \cite{belkin2018overfitting}, with much lower complexity, of an order of magnitude comparable to  KNN-based schemes.
\begin{figure*}[htbp]
    \centering
    \begin{subfigure}{0.35\textwidth}
    \centering
    \includegraphics[width=0.9\textwidth]{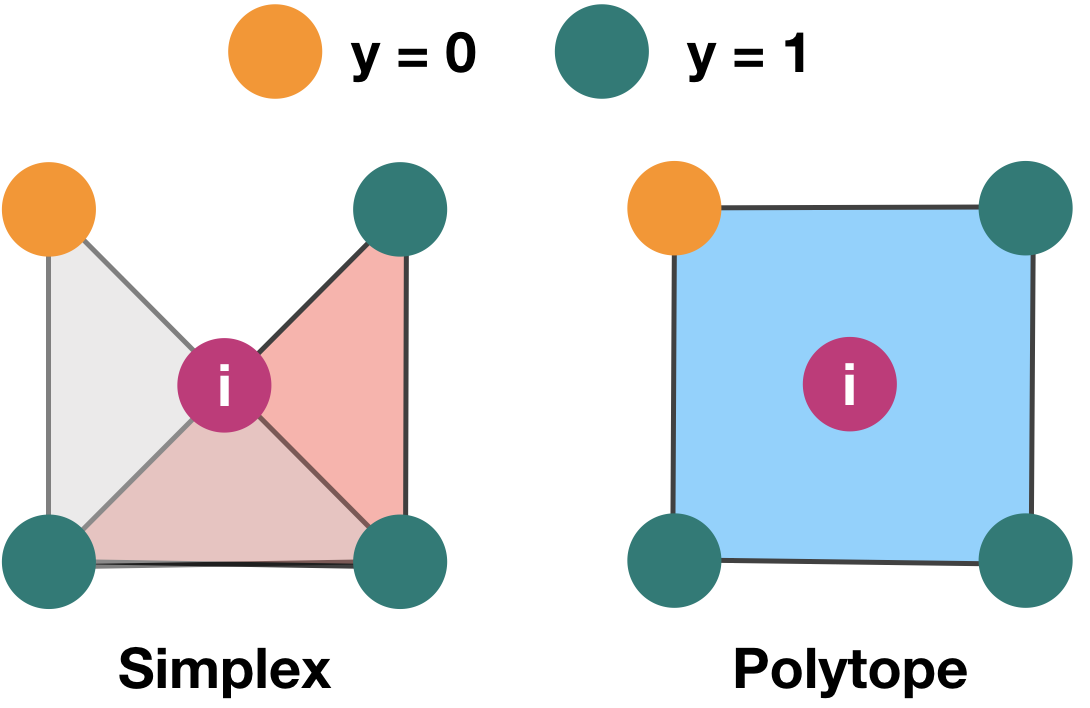}
    \caption{}
    \label{fig:interpolation_difference}
    \end{subfigure}
    \begin{subfigure}{0.33\textwidth}
    \centering
    \includegraphics[width=0.9\textwidth]{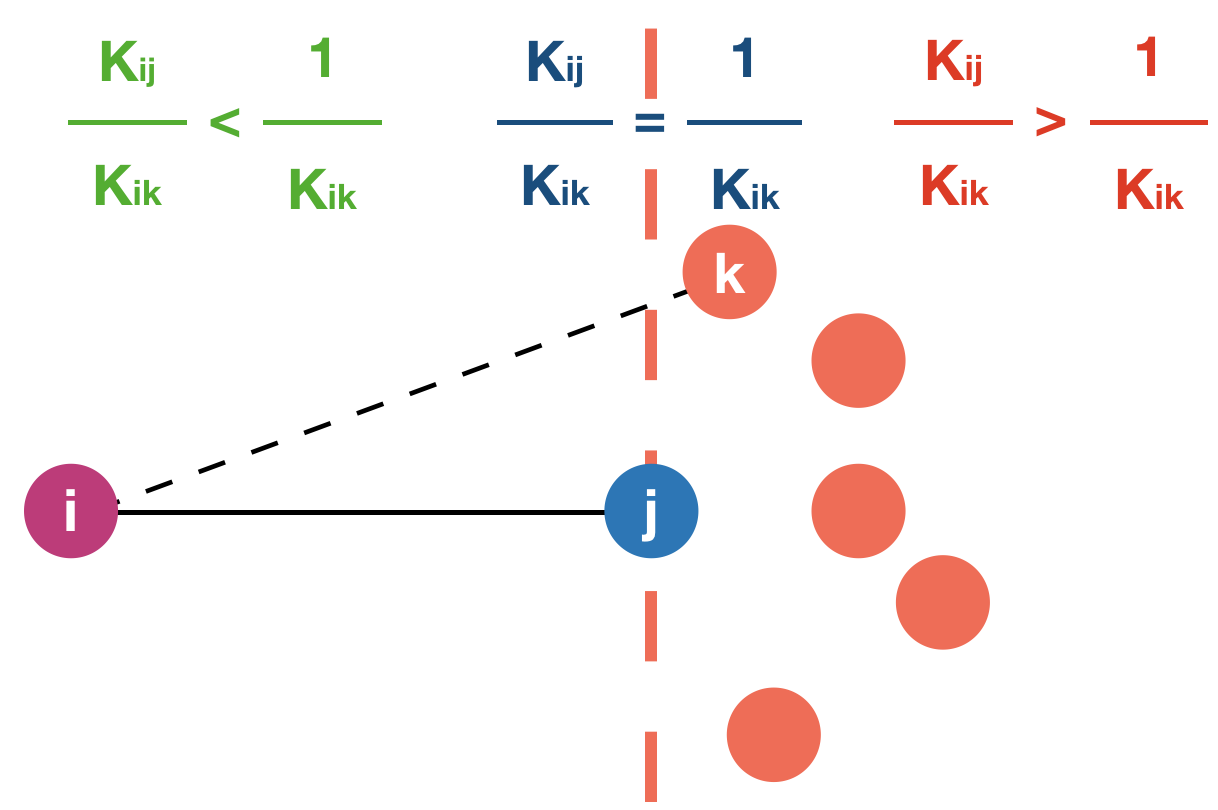}
    \caption{}
    \label{fig:KRI_plane}
    \end{subfigure}
    \begin{subfigure}{0.25\textwidth}
    \centering
    \includegraphics[width=0.9\textwidth]{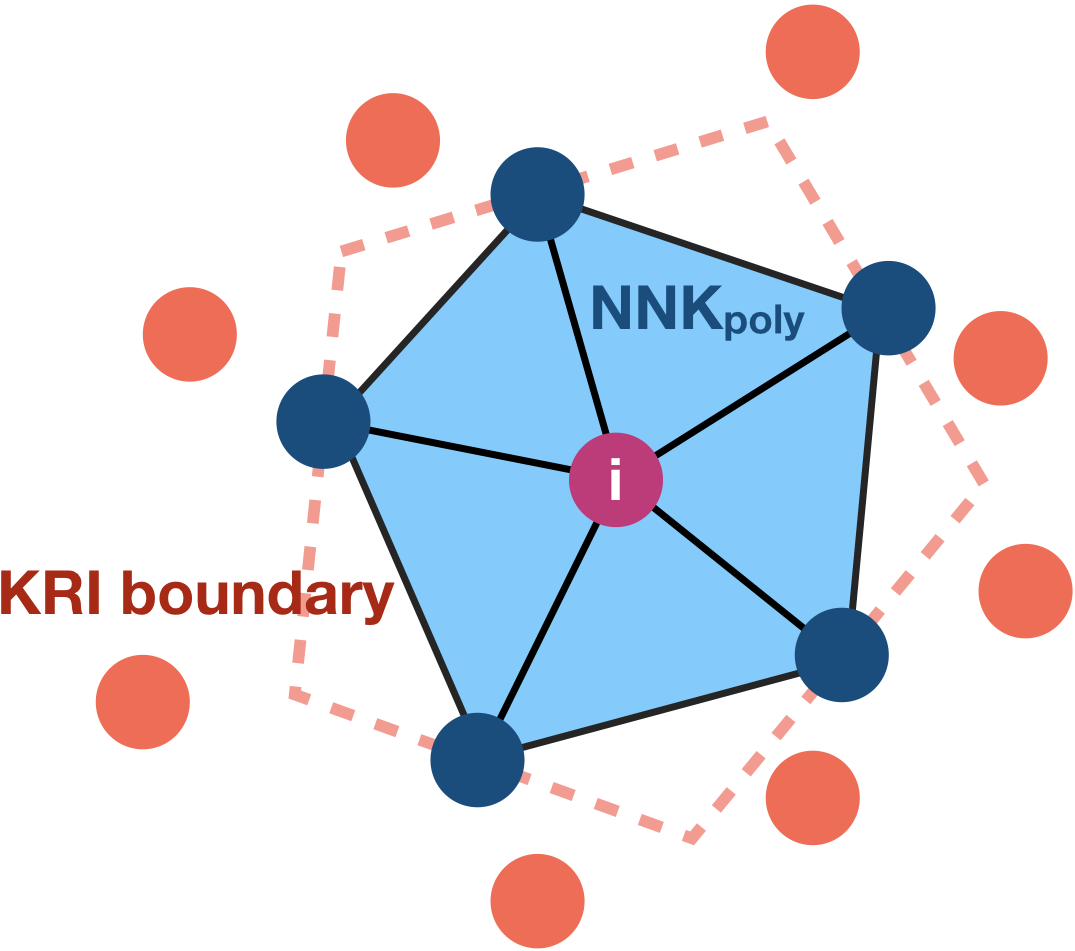}
    \caption{}
    \label{fig:KRI_polytope}
    \end{subfigure}
    \caption{(a) Comparison of simplicial and polytope interpolation methods. In the simplex case, the label for  node $\vx_i$ can be approximated based on different triangles (simplices), one of which must be chosen. With the chosen triangle two out of the three points are used for interpolation, so that in this example only half the neighboring points are used for interpolation. 
    Instead, polytope interpolation based on NNK is based on all four data points, which together form a polytope. (b) KRI plane (dashed orange line) corresponding to chosen neighbor $\vx_j$. Data points to the right of this plane will be not be selected by NNK as neighbors of $\vx_i$. (c) KRI boundary associated convex polytope formed by NNK neighbors at $\vx_i$. }
\end{figure*}

To integrate our interpolation framework with a neural network, we replace 
the final classification layer, typically some type of  support vector machine (SVM) with our NNK interpolator during evaluation at training and at test time, 
while relying on the loss obtained with the original SVM-like layer for backpropagation.
This strategy of using a different classifier at final layer is not uncommon in deep learning\cite{koh2017understanding, lundberg2017unified, bontonou2019introducing} and is motivated by the intuition that each layer of a neural network corresponds to an abstract transformation of the input data space catered to the machine task at hand.
Note that, unlike the explicit parametric boundaries defined in general by an SVM-like final layer, local interpolation methods produce boundaries that are implicit, i.e., based on the relative positions of the training data in a transformed space. 
In other words, the proposed DeepNNK procedure allows us to characterize the network by the output classification space rather than relying on a global boundary defined on the space.

Equipped with these tools, we tackle model selection in neural networks from an interpolation perspective using data dependent stability: A model is stable for training set $\gD$ if any change of a single point in $\gD$ does not affect (or yields very small change in) the output hypothesis \cite{devroye1979distribution, yu2013stability}. 
This definition is similar but different from algorithmic stability obtained using jackknifing \cite{bousquet2002stability, mukherjee2006learning} and related statistical procedures such as cross validation \cite{anders1999model}. While the latter is related to using repeatedly the entire training dataset but one for computing many estimators that are combined at the end, the former is concerned with the output estimate at a point not used for its prediction and is the focus of study in our work.
Direct evaluation of algorithmic stability in the context of deep learning is impractical for two reasons: First, the increased runtime complexity associated with training the algorithm for different sets. Second, even if computationally feasible, the assessment within each setting is obscured due to randomness in training, for example in weight initialization and batch sampling, which requires repeated evaluation to reduce variance in the performance.
Unlike these methods \cite{anders1999model, bousquet2002stability}, by focusing on stability to input perturbations at interpolation, our method achieves a practical methodology not involving repetitive training for model selection.

Another challenging issue that prevents the application of deep neural networks in sensitive domains, such as medicine and defense, is the absence of explanations to a prediction obtained\cite{doshi2017towards}. 
Explainability or interpretability can be defined as the degree to which a human can understand or rationalize a prediction obtained from a learning algorithm. 
A model is more interpretable than another model if its decisions are easier for a human to comprehend, for e.g., a health care technician looking at a flagged retinal scan\cite{de2018clinically}, than decisions from the other model.
Example based explanations can help alleviate the problem of interpretability by allowing humans to understand complex predictions by analogical reasoning with a small set of influential training instances \cite{kim2016examples, koh2017understanding}.

Our proposed DeepNNK classifier is a neighborhood based approach that makes very few modeling assumptions on data. Each prediction in NNK classifier comes with a set of training data points (neighbors selected by NNK) that interpolate to produce the classification/regression. 
In contrast to earlier methods such as DkNN\cite{Papernot2018, wallace2018interpreting} that rely on hyperparameters such as $k, \epsilon$ which directly impact explainability and confidence characterizations, our approach adapts to the local data manifold by identifying a stable set of training instances that most influence an estimate. 
Further, the gains in interpretability using our framework do not incur a penalty in performance,
so that, unlike earlier methods \cite{koh2017understanding, Papernot2018}, there is no loss in overall performance by using an interpolative last layer, and some cases there are gains, as compared to the the performance achieved with standard SVM-like last layer classifiers. Indeed, we demonstrate performance improvements over standard architectures with SVM-like last layers in case where there is overfitting. 

Finally, this paper presents some empirical explanation to generative and adversarial examples, which have gained growing attention in modern machine learning. We show that these instances fall in distinct interpolation regions surrounded by fewer NNK neighbors on average compared to real images.


\section{Preliminaries and Background}
\subsection{Statistical Setup}
The goal of machine learning is to find a function $\hat{f}: X \rightarrow Y$ that minimizes the probability of error on samples drawn from the joint distribution over $X \times Y$ in $\R^d \times [0, 1]$. Assume $\mu$ to be the marginal distribution of $X \in \R^d$ with its support denoted as $supp(\mu)$. Let $\eta$ denote the conditional mean $\E(Y|X=\vx)$. 
The risk or error associated with a predictor in a regression setting is given by $\gR_{gen}(\hat{f}) = \E[\gR(\hat{f}, \vx)] = \E[(\hat{f}(\vx) - y)^2]$. 
The Bayes estimator obtained as the expected value of the conditional is the best predictor and upper bounds other predictors as $\E[\gR(\hat{\eta}, \vx) - \gR(\eta, \vx)] \leq \E[(\hat{\eta}(\vx) - \eta(\vx))^2]$. 
Unfortunately, the joint distribution is not known \textit{a priori} and thus a good estimator is to be designed based on labelled samples drawn from $X \times Y$ in the form of training data $D_{train} = \{(\vx_{1}, y_1), (\vx_{2}, y_2) \dots (\vx_N, y_N)\}$. Further, assume each $y_i$ is corrupted by i.i.d.~noise and hence can deviate from the Bayes estimate $\eta\{\vx_i\}$.
For a binary classification problem, the domain of $Y$ is reduced to $\{0,1\}$, with the plug-in Bayes classifier defined as $f^* = \sI(\eta(\vx)>1/2)$ where $\eta(\vx) = P(Y=1|X=\vx)$. The risk associated to a classifier is defined as $\gR_{gen}(\hat{f}) = \E[\gR(\hat{f}, \vx)] = \E[P(\hat{f}(\vx) \neq y)]$ and is related to the Bayes risk as $\E[\gR(\hat{f}, \vx) - \gR(f^*(\vx), \vx)] \leq \E[P(\hat{f}(\vx) \neq f^*(\vx))]$.
Note that the excess risk associated to $\hat{f}$ in both regression and classification setting is related to $\E[(\hat{\eta}(\vx) - \eta(\vx))^2]$ and $\E[P(\hat{f}(\vx) \neq f^*(\vx))]$, and is the subject of our work. Note that the generalization risk defined above is dependent on the data distribution while in practice one uses empirical error, defined as $\gR_{emp}(\gD_{train}) = \frac{1}{N} \sum_i l(\hat{\eta}(\vx_i),y)$ where $l(\hat{\eta}(\vx_i), y)$ is the error associated in regression or classification setting.
We denote by $\gD_{train}^i$ the training set obtained by removing the point $(\vx_i, y_i)$ from $\gD_{train}$.

\subsection{Deep Kernels}
Given data $\gD = \{\vx_1, \vx_2 \dots \vx_N\}$, kernel based methods observe similarities in a non linearly transformed feature space $\gH$ referred to as the Reproducing Kernel Hilbert Space (RKHS)\cite{aronszajn1950theory}. One of the key ingredients in kernel machines is the \emph{kernel trick}: Inner products in the feature space can be efficiently computed using kernel functions. Due to the non linear nature of the data mapping, linear operations in RKHS correspond to non linear operations in the input data space.

\begin{definition}
If $\gK: \R^d \times \R^d \rightarrow \R$ is a continuous symmetric kernel of a positive integral operator in $\gL_2$ space of functions,
then there exists a space $\gH$ and mapping $\vphi: \R^d \rightarrow \gH$ such that by Mercer's theorem 
\begin{align*}
    \gK(\vx_i, \vx_j) = \langle\vphi(\vx_i), \vphi(\vx_j)\rangle
\end{align*}
where $\langle \cdot , \cdot \rangle$ denotes the inner product.
\end{definition}
Kernels satisfying above definition are known as Mercer kernels and have wide range of  applications in machine learning \cite{hofmann2008kernel}. In this work, we center our experiments around the range normalized cosine kernel defined as,
\begin{align}
    \gK(\vx_i, \vx_j) = \frac{1}{2} \left(1 + \frac{\langle\vx_i, \vx_j\rangle}{\|\vx_i\|\;\|\vx_j\|}\right) \label{eq:base_cosine_kernel}
\end{align}
though our theoretical statements and claims make no assumption on the type of kernel, other than it be positive with range $[0,1]$.
Similar to \cite{wilson2016deep}, we combine kernel definitions with neural networks to incorporate the expressive power of neural networks. Given a kernel function $\gK$, we transform the input data using the non linear mapping $\vh_\vw$ corresponding to deep neural networks (DNN) parameterized by $\vw$. 
\begin{align}
    \gK(\vx_i, \vx_j) \rightarrow \gK_{DNN}(\vh_\vw(\vx_i), \vh_\vw(\vx_j))
\end{align}
Our normalized cosine kernel of \eqref{eq:base_cosine_kernel} is rewritten as
\begin{align}
    \gK_{DNN}(\vx_i, \vx_j) = \frac{1}{2} \left(1 + \frac{\langle\vh_\vw(\vx_i), \vh_\vw(\vx_j)\rangle}{\|\vh_\vw(\vx_i)\|\;\|\vh_\vw(\vx_j)\|}\right) \label{eq:deep_cosine_kernel}
\end{align}

\subsection{Non Negative Kernel regression (NNK)}
The starting point for our interpolation-based classifier is our previous work  on graph construction using non negative kernel regression (NNK) \cite{shekkizhar2020}. NNK formulates graph construction as a signal representation problem, where each data point is to be approximated by a weighted sum of functions from a dictionary formed by its neighbors. The NNK objective for graph construction can be written as:
\begin{align}
 \min_{\vtheta \geq 0} \;
 \|\vphi_i - \mPhi_S\vtheta\|^2 \label{eq:nnk_lle_objective}
\end{align}
where $\vphi_i$ is a lifting of $\vx_i$ from observation to similarity space and $\mPhi_S$ contains the transformed neighbors.

Unlike $k$ nearest neighbor approaches, which select neighbors having the $k$ largest inner products  $\vphi_i^\top\vphi_j$ and can be viewed as a thresholding-based representation, NNK is an improved basis selection procedure in kernel space leading to a stable and robust representation. 
Geometrically, NNK can be characterized in the form of kernel ratio interval (KRI) as shown in \twoFigref{fig:KRI_plane}{fig:KRI_polytope}. The KRI theorem states that for any positive definite kernel with range in $[0, 1]$ (e.g. the cosine kernel \plainref{eq:deep_cosine_kernel}), the necessary and sufficient condition for two data points $\vx_j$ and $\vx_k$ to be {\em both} NNK neighbors to $\vx_i$ is
\begin{align}
\mK_{j,k} < \frac{\mK_{i,j}}{\mK_{i,k}} < \frac{1}{\mK_{j,k}}. 
\label{eq:kernel_ratio_interval}
\end{align}
Inductive application of the KRI produces a closed decision boundary around the data to be approximated ($\vx_i$) with the identified neighbors forming a convex polytope around the data ($NNK_{poly}$). 
Similar to the simplicial interpolation of \cite{belkin2018overfitting}, the local geometry of our NNK  classifier can be leveraged to obtain theoretical performance of bounds as discussed next. 

\section{Local Polytope Interpolation}
\label{sec:polytop_interpolation}
In this section, we propose and analyze a polytope interpolation scheme based on local neighbors\footnote{All proofs related to theoretical statements in this section are included in the supplementary material} that asymptotically approaches the $1$-nearest neighbor algorithm \cite{cover1967nearest}. Like $1$-nearest neighbor, the proposed method is not statistically consistent in the presence of label noise, but, unlike the former, it's risk can be studied in the non-asymptotic case with data dependent bounds under mild assumptions on smoothness. 
\begin{proposition}
\label{prop:nnk_interpolation}
Given $k$ nearest neighbors of a sample $\vx$, $S = \{(\vx_{1}, y_1), (\vx_{2}, y_2) \dots (\vx_{k}, y_k)\}$, the following NNK estimate at $\vx$ is a valid interpolation function:
\begin{align}
\hat{\eta}(\vx) = \E(Y | X = \vx) = \sum_{i=1}^{\hat{k}}\vtheta_{i}\;y_i,  \label{eq:biased_nnk_conditional_estimate}
\end{align}
where $\vtheta$ are the $\hat{k}$ non zero weights obtained from the minimization of 
 \eqref{eq:nnk_lle_objective}, that is: 
\begin{align}
\vtheta &= \min_{\vtheta \geq 0} \|\vphi(\vx) - \mPhi_S\vtheta\|^2 \nonumber \\
&= \min_{\vtheta \geq 0}\; 1 - 2\vtheta^\top\mK_{S,*} + \vtheta^\top\mK_{S,S}\vtheta \label{eq:nnk_kernel_objective}
\end{align}
where $\mPhi_S = [\vphi(\vx_1) \dots \vphi(\vx_k)]$ corresponds to the kernel space representation of the nearest neighbors with $\mK_{S,*}$ denoting the kernel similarity with regards to $\vx$.
\end{proposition}
The interpolator from Proposition \ref{prop:nnk_interpolation} is biased and can be bias-corrected by normalizing the interpolation weights.
Thus, the unbiased NNK interpolation estimate is obtained as 
\begin{align}
\hat{\eta}(\vx) = \E(Y | X = \vx) = \sum_{i=1}^{\hat{k}}\frac{\vtheta_{i}}{\sum_{j=1}^{\hat{k}} \vtheta_j}\;y_i \label{eq:nnk_conditional_estimate}
\end{align}
In other words, NNK starts with a crude approximation of neighborhood in the form of $k$ nearest neighbors, but instead of directly using these points as sources of interpolation, optimizes and reweighs the selection (most of which are zero) using \eqref{eq:nnk_kernel_objective} to obtain a stable set of neighbors.
\subsection{A general bound on NNK classifier}
We present a theoretical analysis based on the simplicial interpolation analysis by \cite{belkin2018overfitting} but adapted to the proposed NNK interpolation. We first study NNK framework in a regression setting and then adapt the results for classification. Let $D_{train} = \{ (\vx_{1}, y_1), (\vx_{2}, y_2) \dots (\vx_N, y_N)\}$ in $\R^d\times [0,1]$ be the training data made available to NNK. Further, assume each $y_i$ is corrupted by independent noise and hence can deviate from the Bayes estimate $\eta(\vx_i)$.
\begin{theorem}
\label{thm:excess_mean_sq_risk}
For a conditional distribution $\hat{\eta}(\vx)$ obtained using unbiased NNK interpolation given training data $D_{train} = \{ (\vx_{1}, y_1), (\vx_{2}, y_2) \dots (\vx_N, y_N)\}$ in $\R^d\times [0, 1]$, the excess mean square risk is given by
\begin{align}
    \E[(\hat{\eta}(\vx) - \eta(\vx))^2 | D_{train}]  \leq \E[\mu(\R^d \backslash \gC)] + A^2\E[\delta^{2\alpha}] \nonumber\\
    + \frac{2A'}{\E_K[\hat{k}]+1}\E[\delta^{\alpha'}] + \frac{2}{\E_K[\hat{k}]+1}\E[(Y - \eta(\vx))^2] \label{eq:excess_sq_risk}
\end{align}
under the following assumptions
\begin{enumerate}
    \item $\mu$ is the marginal distribution of $X \in \R^d$. Let $\gC = \text{Hull}(\vphi(\vx_1), \vphi(\vx_2) \dots \vphi(\vx_N))$ be the convex hull of the training data  in transformed kernel space.
    \item The conditional distribution $\eta$ is Holder $(A, \alpha)$ smooth in kernel space.
    \item Similarly, the conditional variance $var(Y|X=\vx)$ satisfies $(A', \alpha')$ smoothness condition.
    \item Let $NNK_{poly}(\vx)$ denote the convex polytope around $\vx$ formed by $\hat{k}$ neighbors identified by NNK with non zero weights. The maximum diameter of the polytope formed with NNK neighbors for any data in $\gC$ is represented as $\delta = \max_{\vx \in \gC} \text{diam}(NNK_{poly}(\vx))$.
\end{enumerate}
\end{theorem}
\begin{remark}
Theorem \plainref{thm:excess_mean_sq_risk} provides a non-asymptotic upper bound for the excess squared risk associated with unbiased NNK interpolation using a data dependent bound. The first term in the bound is associated to extrapolation, where the test data falls outside the interpolation area for the given training data while the last term corresponds to label noise. Of interest are the second and third terms, which merely reflect the dependence of the interpolation on the size of each polytope defined for test data and the associated smoothness of the labels over this region of interpolation. In particular, when all test samples are covered by a smaller polytope, the corresponding risk is closer to optimal. 
Note that NNK approach leads to a polytope having smallest diameter or volume for the number of points ($\hat{k}$) selected from a set of $k$ neighbors. 
From the theorem, this corresponds to a better risk bound. 
The bound associated with simplicial interpolation is a special case, where each simplex enclosing the data point is a fixed $\hat{k}$, corresponding to a $(d+1)$-sized polytope. Thus, in our approach the number of points forming the polytope is variable (dependent on local data topology), while in the simplicial case it is fixed and depends on the dimension of the space. 
Though the latter bound seems better (excess risk is inversely related to $\hat{k}$), the diameter of the polytope~(simplex) increases with $d$ making the excess risk possibly sub optimal.
\end{remark}
\begin{corollary}
\label{coroll:excess_mean_sq_risk_convergence}
Based on an additional assumption that $supp(\mu)$ belongs to a simple polytope, the excess mean square risk converges asymptotically as 
\begin{align}
    \limsup_{N\rightarrow\infty} \E[(\hat{\eta}(\vx) - \eta(\vx))^2] \leq \E[(Y - \eta(\vx))^2]
\end{align}
\end{corollary}
\begin{remark}
The asymptotic risk of proposed NNK interpolation method is bounded like the $1$-nearest neighbor method in the regression setting by twice the Bayes risk.  The rate of convergence of proposed method is dependent on the convergence of the kernel functions centered at the data points. 
\end{remark}
We now extend our analysis to classification using the plug-in classifier $\hat{f}(\vx) = \sI(\hat{\eta}(\vx) > 1/2)$ for a given $D_{train} = \{ (\vx_{1}, y_1), (\vx_{2}, y_2) \dots (\vx_N, y_N)\}$ in $\R^d\times \{0,1\}$ using the relationship between classification and regression risk \cite{biau2015lectures}. 
\begin{corollary}
\label{coroll:classifier_risk_convergence}
A plug-in NNK classifier under the assumptions of Corollary \ref{coroll:excess_mean_sq_risk_convergence} has excess classifier risk bounded as 
\begin{align}
    \limsup_{N\rightarrow \infty} \E[\gR(\hat{f}(\vx)) - \gR(f(\vx))] \leq 2\sqrt{\E[(Y - \eta(\vx)^2]}
\end{align}
\end{corollary}
\begin{remark}
The classification bound presented here makes no assumptions on the margin associated to the classification boundary and is thus only a weak bound. The bound can be improved exponentially as in \cite{belkin2018overfitting} when more assumptions such as h-hard margin condition \cite{massart2006risk} are made.
\end{remark}
\subsection{Leave one out stability}
\label{subsec:loo_stability}
The leave one out~(LOO) procedure (also known as deleted estimate or U-method) is an important statistical measure with a long history in machine learning \cite{elisseeff2003leave}. Unlike empirical error, it is \emph{almost unbiased} \cite{luntz1969estimation} and has been often used for model (hyperparameter) selection. 
Formally, this is represented by
\begin{align}
    \gR_{loo}(\hat{\eta} | \gD_{train}) = \frac{1}{N}\sum_{i=1}^N l(\hat{\eta}(\vx_i)|\gD_{train}^i, y_i)
\end{align}
where the NNK interpolation estimator in the summation for $\vx_i$ is based on all training points except $\vx_i$.
We focus our attention to LOO in the context of model stability and generalization as defined in \cite{elisseeff2003leave, devroye2013probabilistic}. A system is stable when small perturbations (LOO) to the input data do not affect its output predictions i.e.,
$\hat{\eta}$ is $\beta(\hat{\eta})$ stable when
\begin{align}
    \E_\gD\left(|l(\hat{\eta}(\vx)| \gD, y) - l(\hat{\eta}(\vx)| \gD^i, y)|\right) \leq \beta(\hat{\eta} | \gD) \label{eq:stability_definition} 
\end{align}

Theoretical results by Rogers, Devroye and Wagner \cite{rogers1978finite, devroye1979distribution, devroye1979deleted}  about generalization of $k$-nearest neighbor methods using LOO performance are  very relevant to our proposed NNK algorithm. The choice of $k$ in our method is dependent on the relative positions of points and hence replaces the fixed $k$ from their results by expectation.
\begin{theorem}
\label{thm:loo_bound_theorem}
The leave one out performance of unbiased NNK classifier given $\gamma$, the maximum number of distinct points that can share the same nearest neighbor, is bounded as 
\begin{align*}
P(|\gR_{loo}(\hat{\eta}|\gD_{train}) - \gR_{gen}(\hat{\eta})| > \epsilon) &\leq 2e^{-N\epsilon^2/18} \;+\\
  &6e^{-N\epsilon^3/\left(108\E_K[\hat{k}](2 + \gamma)\right)}
\end{align*}
\end{theorem}
\begin{remark}
NNK classifier weighs its neighbors based on RKHS interpolation but obtains the initial set of neighbors based on the input embedded space. This means the value of $\gamma$ in NNK setting is dependent on the dimension of the space of points where the data is embedded and not on the possibly infinite dimension of the RKHS. The above bound is difficult to compute in practice due to $\gamma$ but bounds do exist for this measure based on convex covering literature \cite{rogers1963covering, chen2005new}. The theorem allows us to relate stability of a model using LOO error to that of generalization. Unlike the bound based on hyperparameter $k$, the bound presented here is training data dependent due to the data dependent selection of neighbors.
\end{remark}
More practically, to characterize the smoothness in the classification surface, we introduce variation or spread in LOO interpolation score of the training dataset as 
\begin{align}
    \nabla(\vx_i) = \frac{1}{\hat{k}}\sum_{j = 1}^{\hat{k}}\left[\hat{\eta}(\vx_i) | \gD_{train}^i - \hat{\eta}(\vx_j)|\gD_{train}^j\right]^2  \label{eq:interpolation_score_spread}
\end{align}
where $\hat{k}$ is the number of non zero weighted neighbors identified by NNK and $\hat{\eta}(\vx)$ is the unbiased NNK interpolation estimate of \eqref{eq:nnk_conditional_estimate}.  A smooth interpolation region will have variation~$\nabla(\vx)$ in its region close to zero while a spread close to one corresponds to a noisy classification region. 
\section{Experiments}
\label{sec:experiments}
In this section, we present an experimental evaluation of DeepNNK for model selection, robustness and interpretability of neural networks.
We focus on experiments with CIFAR-10 dataset to validate our analysis and intuitions on generalization and  interpretability. 
\begin{figure*}[htbp]
    \centering
    \begin{subfigure}{\textwidth}
    \centering
    \includegraphics[trim={0cm, 14cm, 0cm, 0cm}, clip, width=\textwidth]{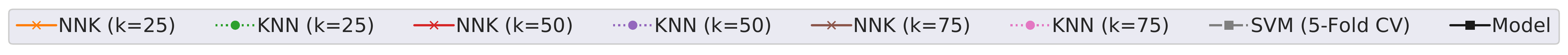}
    \end{subfigure}
    \begin{subfigure}{0.32\textwidth}
    \centering
    \includegraphics[ width=\textwidth]{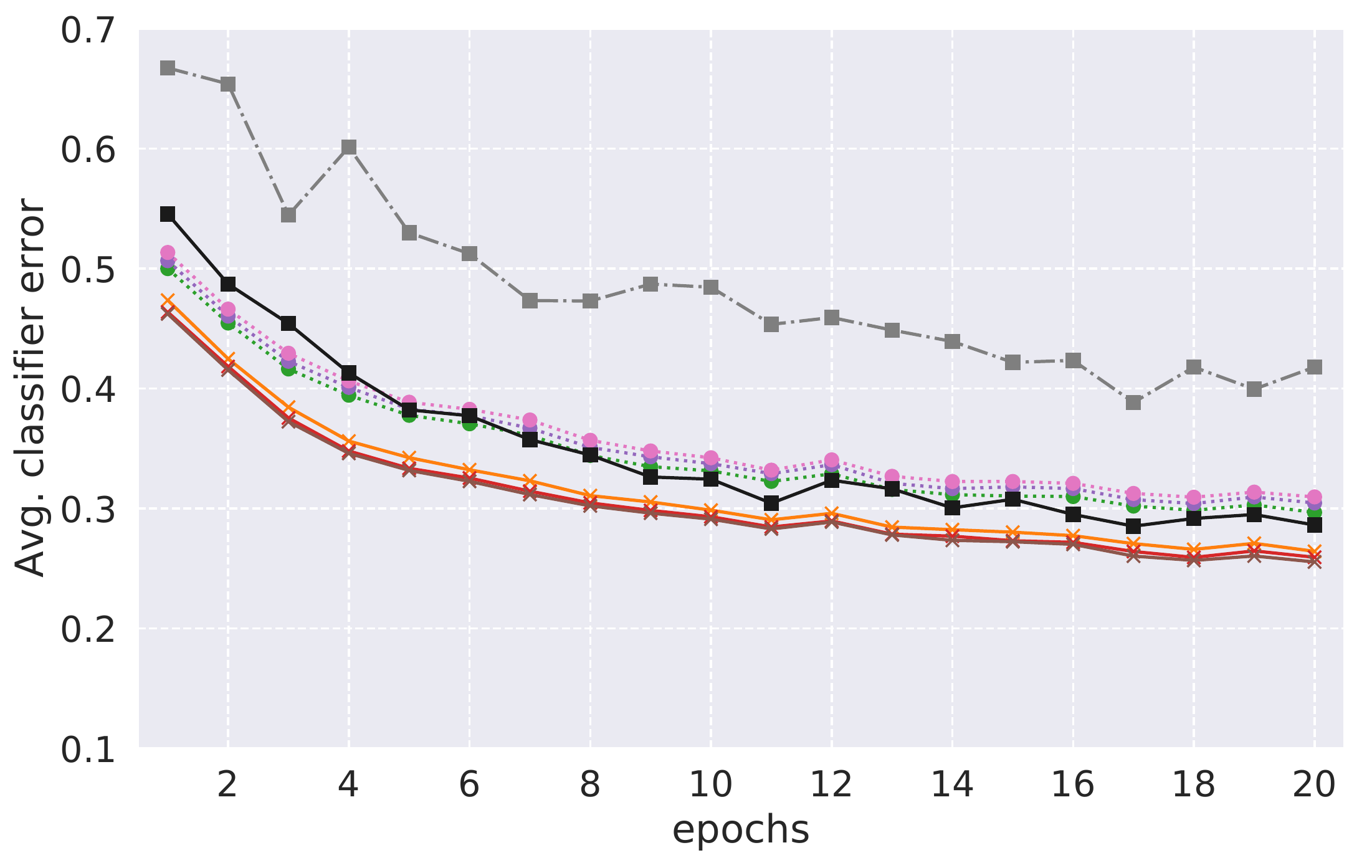}
    \end{subfigure}
    \begin{subfigure}{0.32\textwidth}
    \centering
    \includegraphics[width=\textwidth]{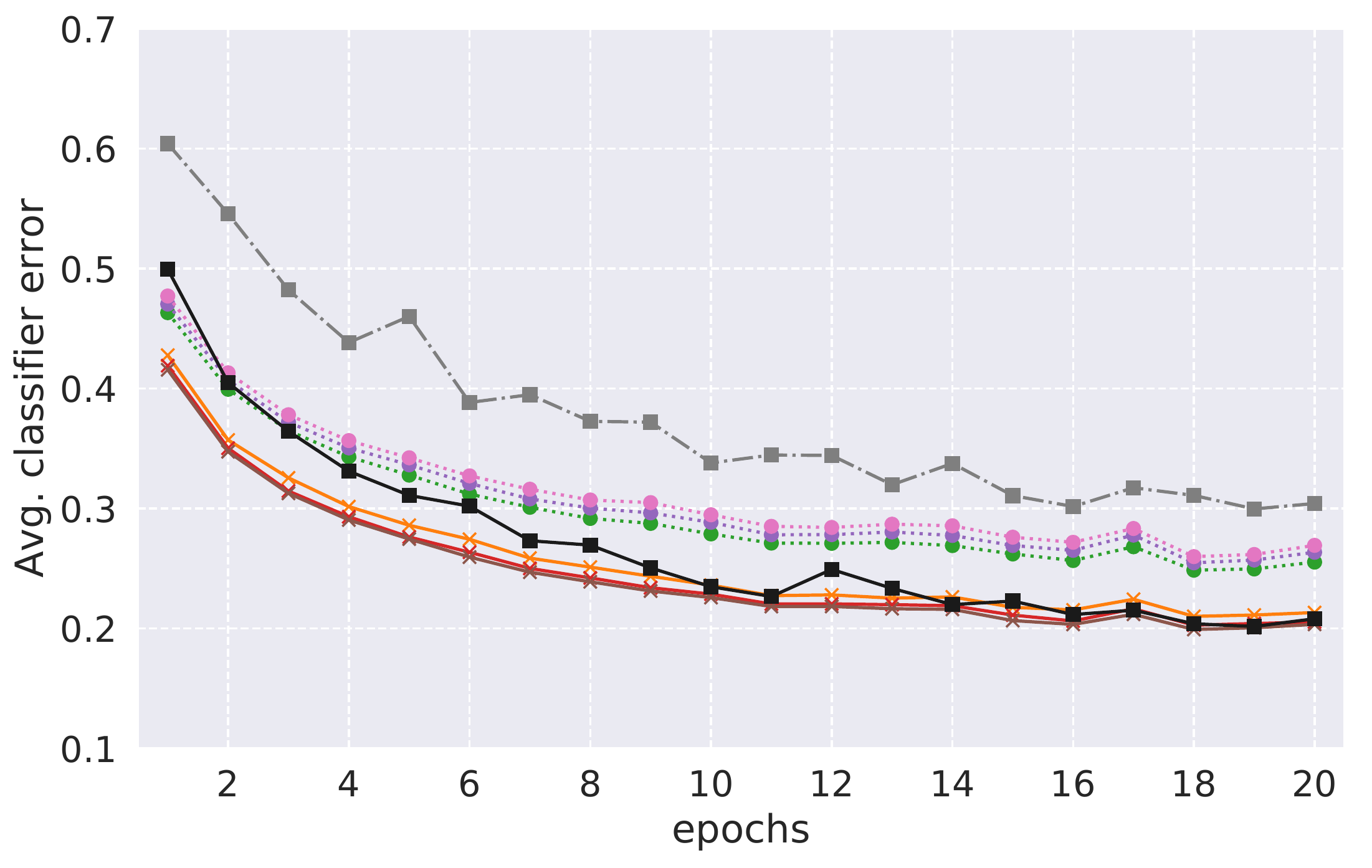}
    \end{subfigure}
    \begin{subfigure}{0.32\textwidth}
    \centering
    \includegraphics[width=\textwidth]{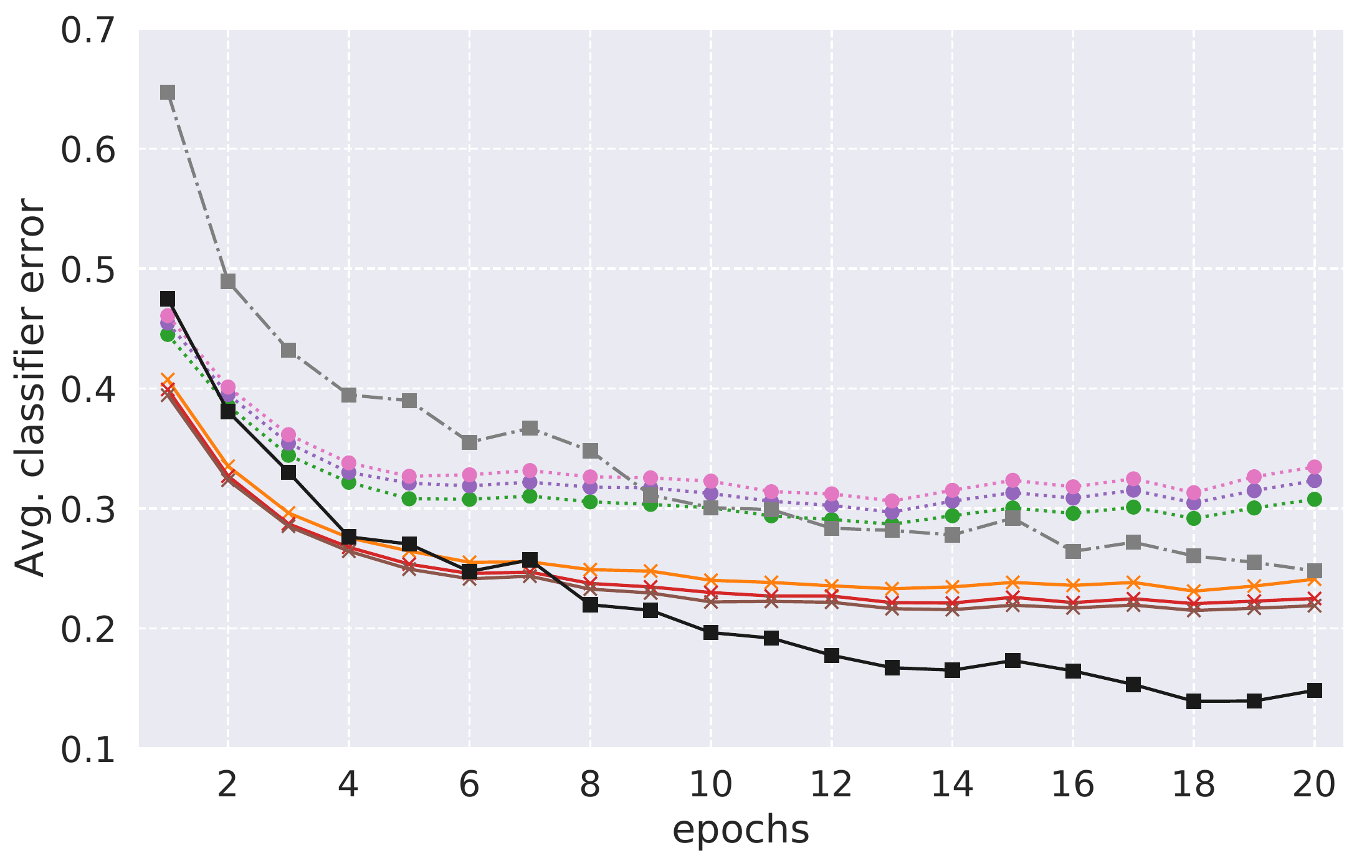}
    \end{subfigure}
    
   \begin{subfigure}{0.32\textwidth}
   \centering
    \includegraphics[width=\textwidth]{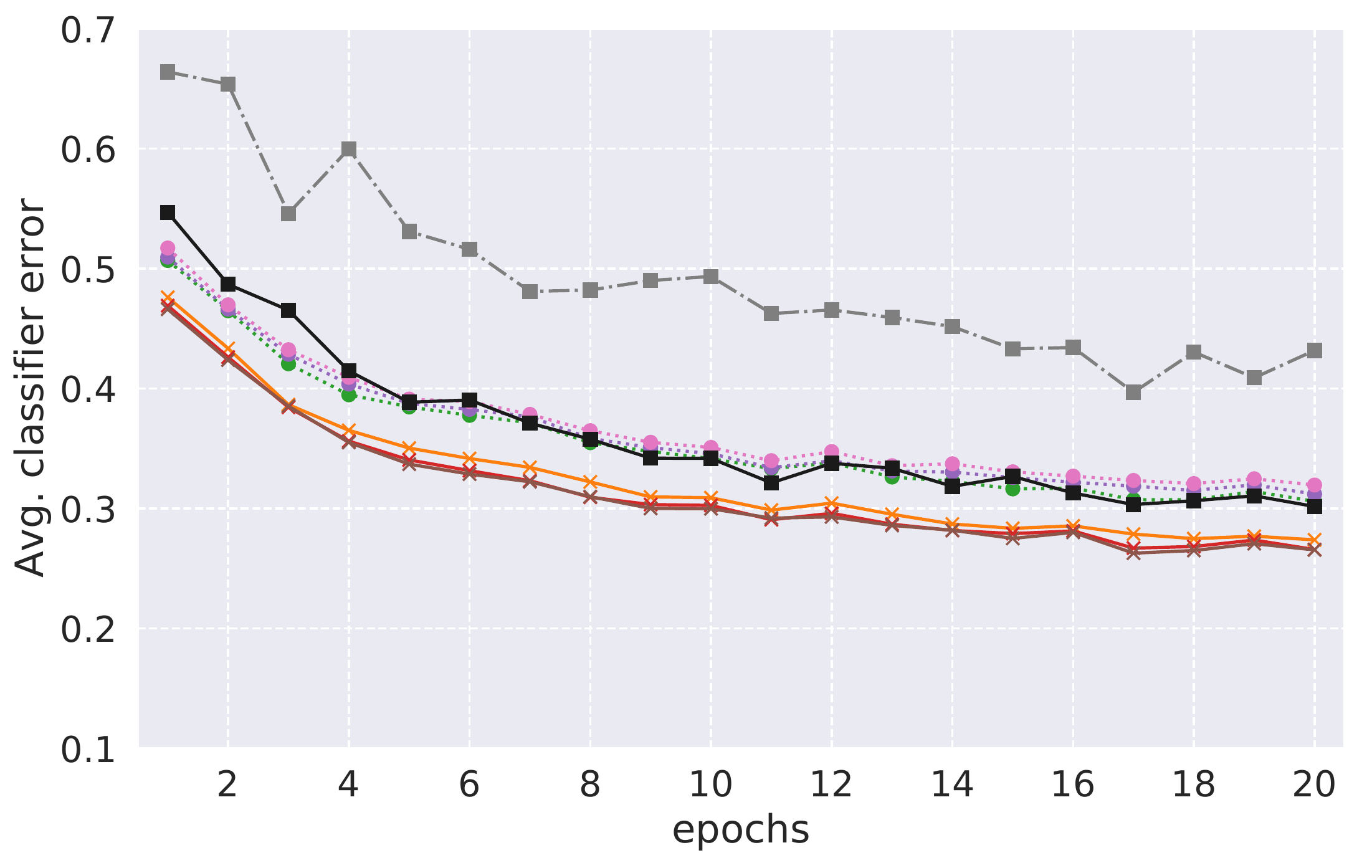}
    \caption{Underparameterized}
    \end{subfigure}
    \begin{subfigure}{0.32\textwidth}
    \centering
    \includegraphics[width=\textwidth]{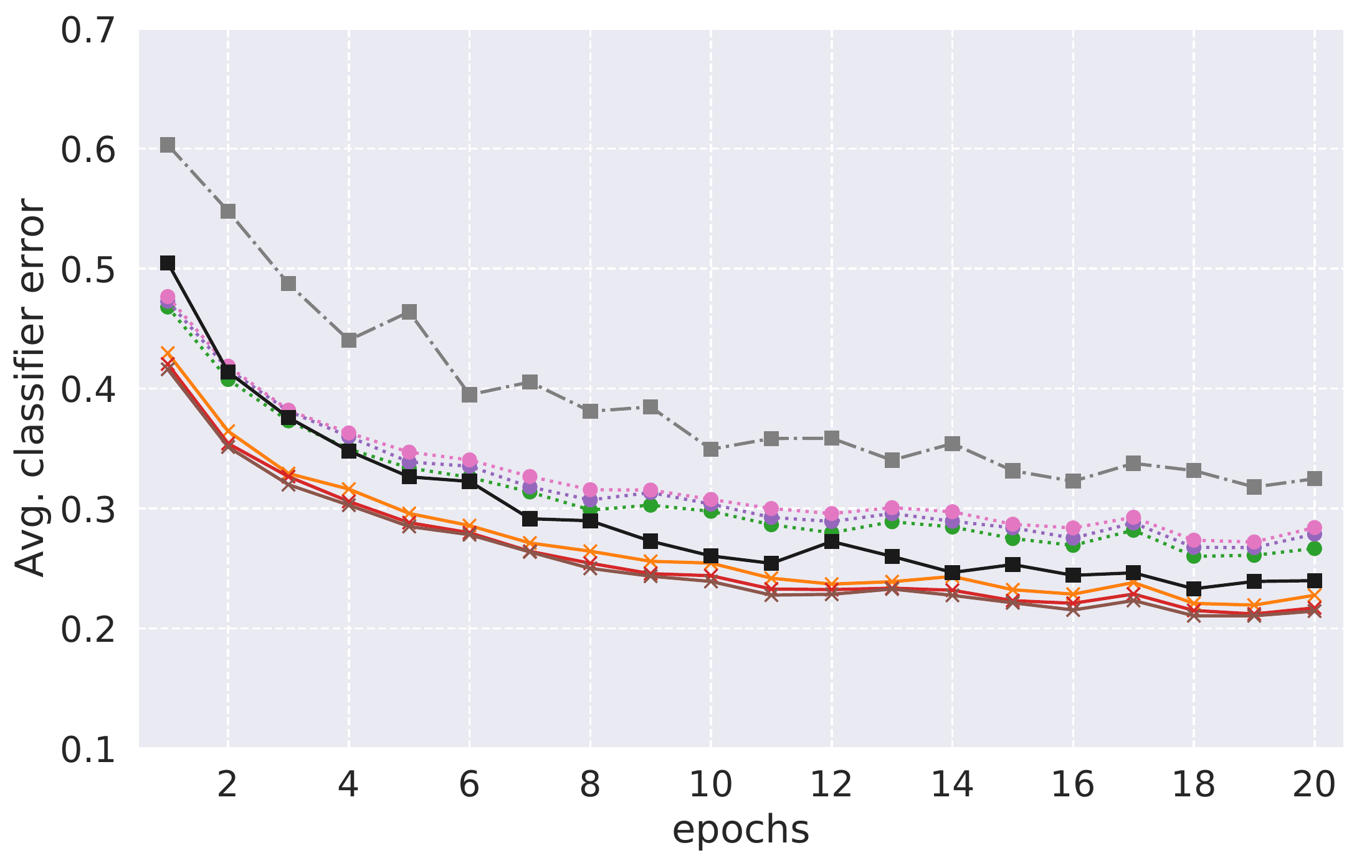}
    \caption{Regularized}
    \end{subfigure}
    \begin{subfigure}{0.32\textwidth}
    \centering
    \includegraphics[width=\textwidth]{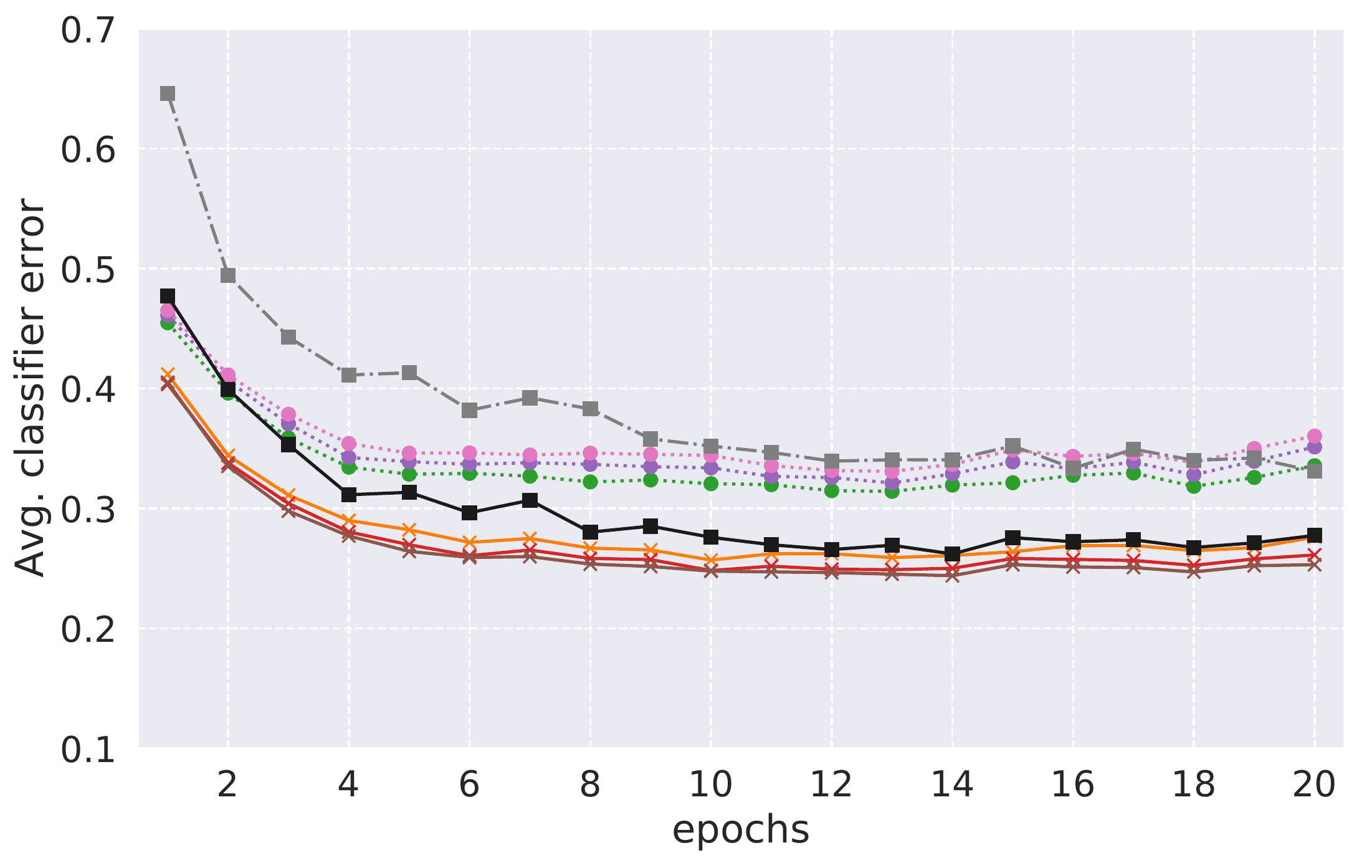}
    \caption{Overfit}
    \end{subfigure}
    \caption{Misclassification error ($\xi$) using fully connected softmax classifier model and interpolating classifiers (weighted KNN, NNK) for different values of $k$ parameter at each training epoch on CIFAR-10 dataset. The training data (Top) and test data (Bottom) performance for three different model settings is shown in each column. NNK classification consistently performs as well as the actual model with classification error decreasing slightly as $k$ increases. On the contrary, a weighted KNN model error increases for increasing $k$ showing robustness issues. The classification error gap between DNN model and leave one out DeepNNK model for train data is suggestive of underfitting ($\xi_{\text{NNK}} < \xi_{\text{model}}$) and overfitting ($\xi_{\text{NNK}} > \xi_{\text{model}}$). We claim a good model to be one where the performance of the model agrees with the local  NNK model.}
    \label{fig:overfitting_study}
\end{figure*}
We consider a simple 7 layer network comprising 4 convolution layers with reLU activations, 2 max-pool layers and 1 full connected softmax layer to demonstrate model selection. We evaluate the test performance and stability of proposed NNK classifier and compare it to weighted KNN (wiNN) approach for different values of $k$ and 5-fold cross validated linear SVM\footnote{Similar to neighborhood methods, the last layer is replaced and trained at each evaluation using a LIBLINEAR SVM \cite{fan2008liblinear} with minimal $\ell_2$ regularization. We use the default library setting for other parameters of the SVM.} for three different network settings, 
\begin{itemize}
    \item Regularized model training: We used 32 depth channels for each convolution layer with dropout at each convolution layer with keep probability 0.9. The data was augmented with random horizontal flips. 
    \item Underparametrized model: We keep the same model structure and regularization as in the regularized model but reduce the number of depth channels to 16, equivalently the number of parameters of the model by half.
    \item Overfit model: To simulate overfitting, we remove data augmentation and dropout regularization in the regularized model while training for the same number of epochs.
\end{itemize}
\begin{figure}[tbp]
    \centering
    \includegraphics[width=0.48\textwidth]{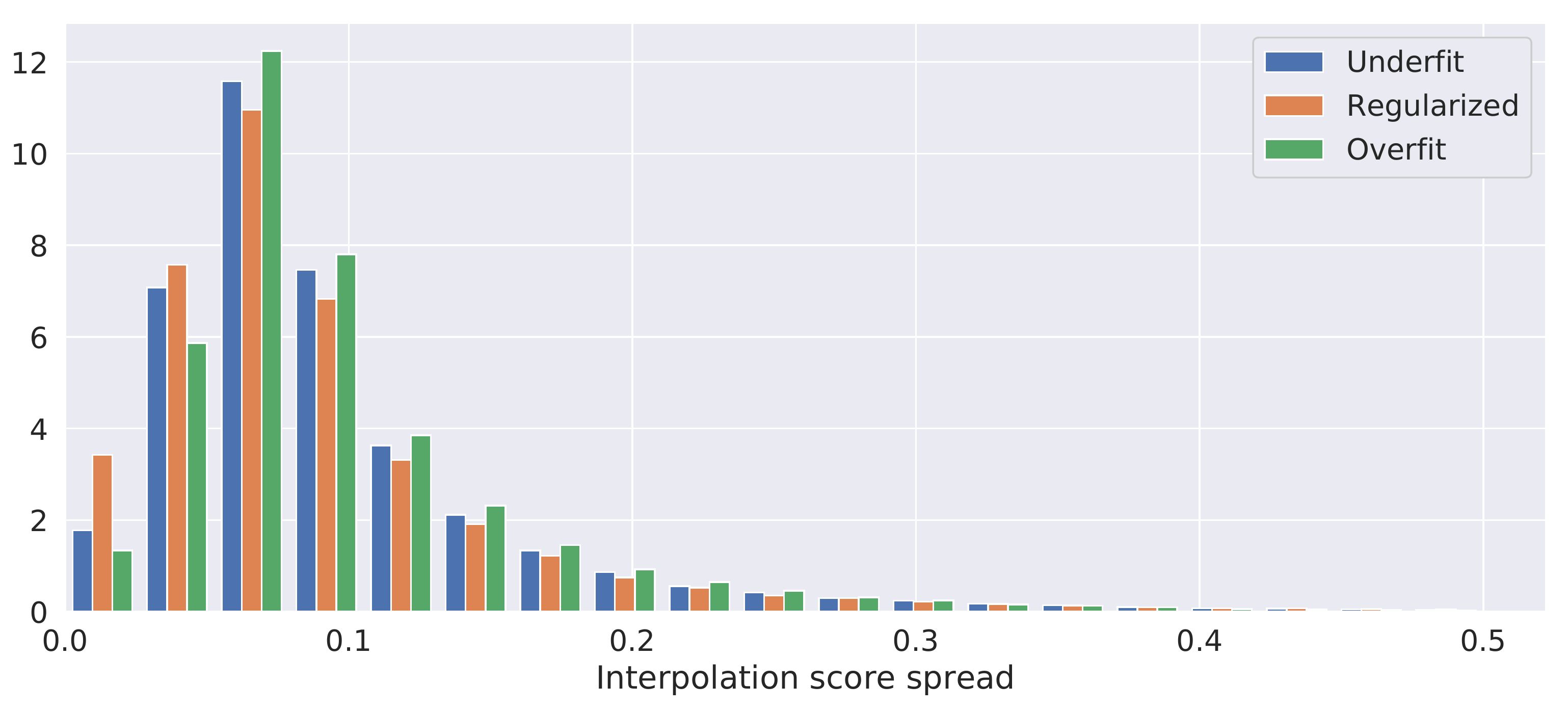}
    \caption{Histogram (normalized) of leave one out interpolation score  after 20 epochs with $k=50$ on CIFAR-10. While the network performance on train dataset is considerably different in each setting, we see that the change in the interpolation~(classification) landscape associated with the input data is minimal which suggests a small change in generalization of the models. The spread is more shifted towards zero in regularized model indicative of a smoother classification surface.}
    \label{fig:intepolation_score_spread}
\end{figure}
\Figref{fig:overfitting_study} shows the difference in performance between our method and weighted KNN (wiNN), in particular, while the proposed DeepNNK method improves marginally with larger values of $k$, the wiNN approach degrades in performance. This can be explained by the fact that NNK accommodates new neighbors only if they belong to a new direction in space that improves its interpolation unlike its KNN counterparts which simply interpolate with \textit{all} $k$ neighbors. More importantly, we observe that while NNK method performs on par if not better than the original classifier with SVM last layer, its LOO performance is a better indicator of the generalization as opposed to the empirical model performance on training data. One can clearly identify a regularized model with better stability by observing the deviation in performance between training and the LOO estimate using our proposed method. Note that the cross validated linear SVM model performed sub-optimally in all settings, which suggests that it is unable to capture the complexity of input data or the generalization difference in different models. The choice of better model is reinforced again in \Figref{fig:intepolation_score_spread}, where we observe that the histogram of interpolation spread for regularized model is more shifted towards zero relative to under-parameterized and overfit models. Note that, the shift is minimal which is expected as the different in test error associated with each model is small as well.

\begin{figure}[htbp]
    \centering
    \begin{subfigure}{0.48\textwidth}
    \includegraphics[width=\textwidth]{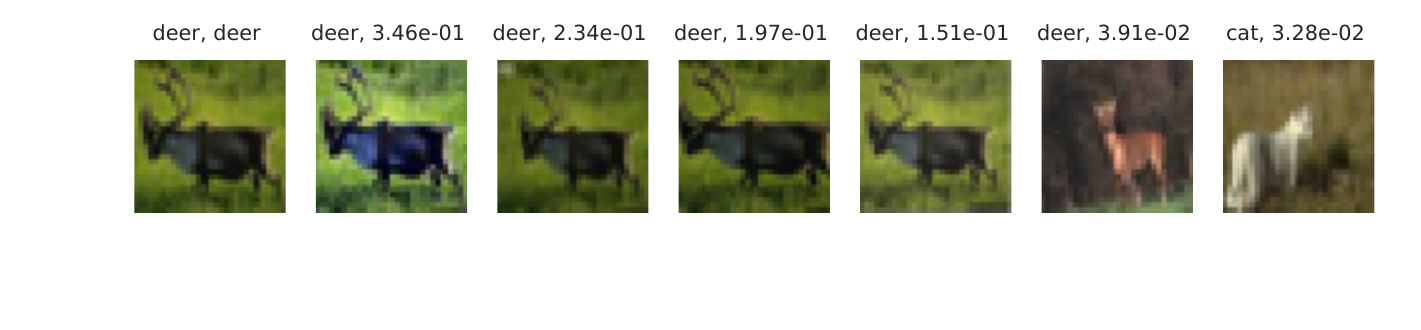}
    \end{subfigure}
    \begin{subfigure}{0.48\textwidth}
    \includegraphics[width=\textwidth]{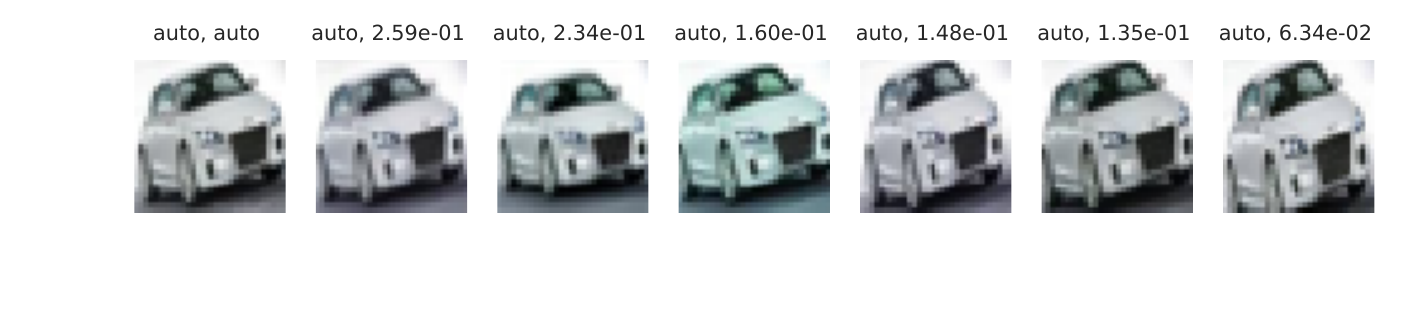}
    \end{subfigure}
    \caption{Two test examples (first image in each set) with identified NNK  neighbors from CIFAR-10 for $k=50$. We show the assigned and predicted label for the test sample, and assigned label and NNK weight for neighboring (and influential) training instances. Though we were able to identify the correct label for the test sample, one might want to question such duplicates in dataset for downstream applications.}
    \label{fig:explainability_example_duplicates}
\end{figure}
We next present a few interpretability results, showing our framework's ability to capture training instances that are influential in prediction. 
Neighbors selected from training data for interpolation by DeepNNK  can be used as examples to explain the neural network decision. This \emph{intepretability} can be crucial to problems with transparency requirements by allowing an observer to interpret the region around a test representation as evidence. 

In \Figref{fig:explainability_example_duplicates}, we show examples in the training dataset that are responsible for a prediction using the simple regularized model defined previously. Machine models and the datasets used for their training often contain biases, such as repeated instances with small perturbations for class balance, which are often undesirable for applications where fairness is important. DeepNNK framework can help understand and eliminate sources of bias by allowing practitioners to identify the limitations of their current system in a semi supervised fashion. 
\Figref{fig:explainability_example_brittleness} shows another application of NNK where the fragile nature of a model over certain training images is brought to light using interpolation spread of \eqref{eq:interpolation_score_spread}. These experiments show the possibility of DeepNNK framework being used as a debugging tool in deep learning. 
\begin{figure}[tbp]
    \centering
    \begin{subfigure}{0.48\textwidth}
    \includegraphics[trim={0cm, 0cm, 0cm, 0.2cm}, clip,width=\textwidth]{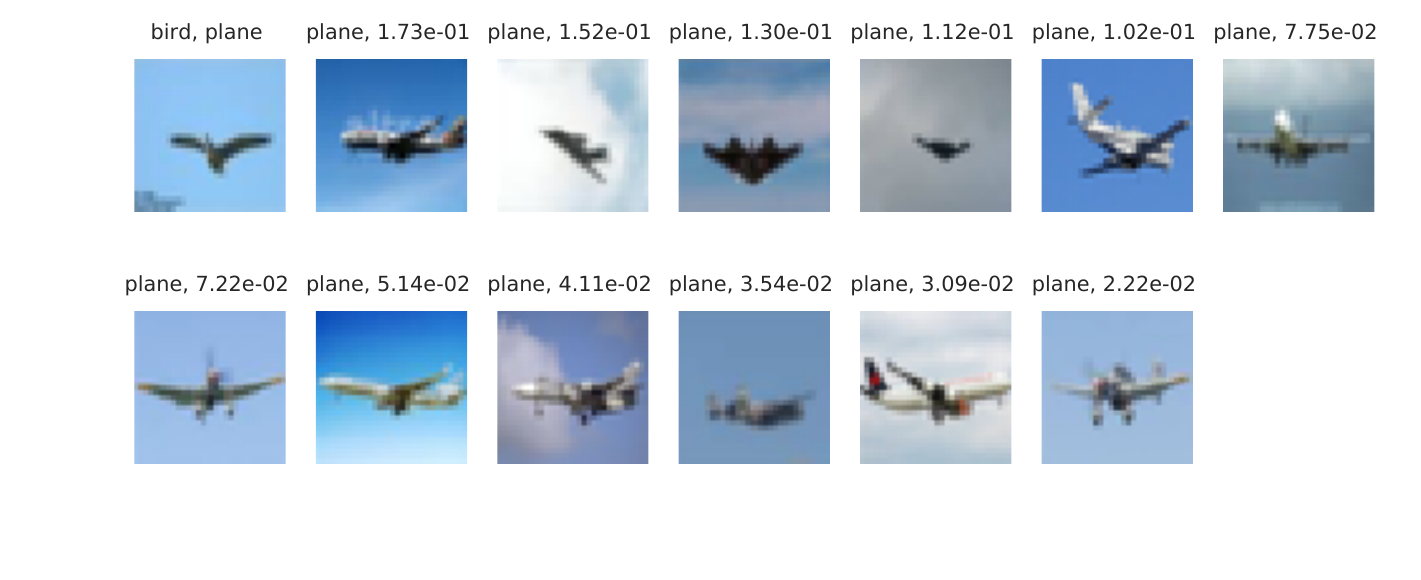}
    \end{subfigure}
    \begin{subfigure}{0.48\textwidth}
    \includegraphics[width=\textwidth]{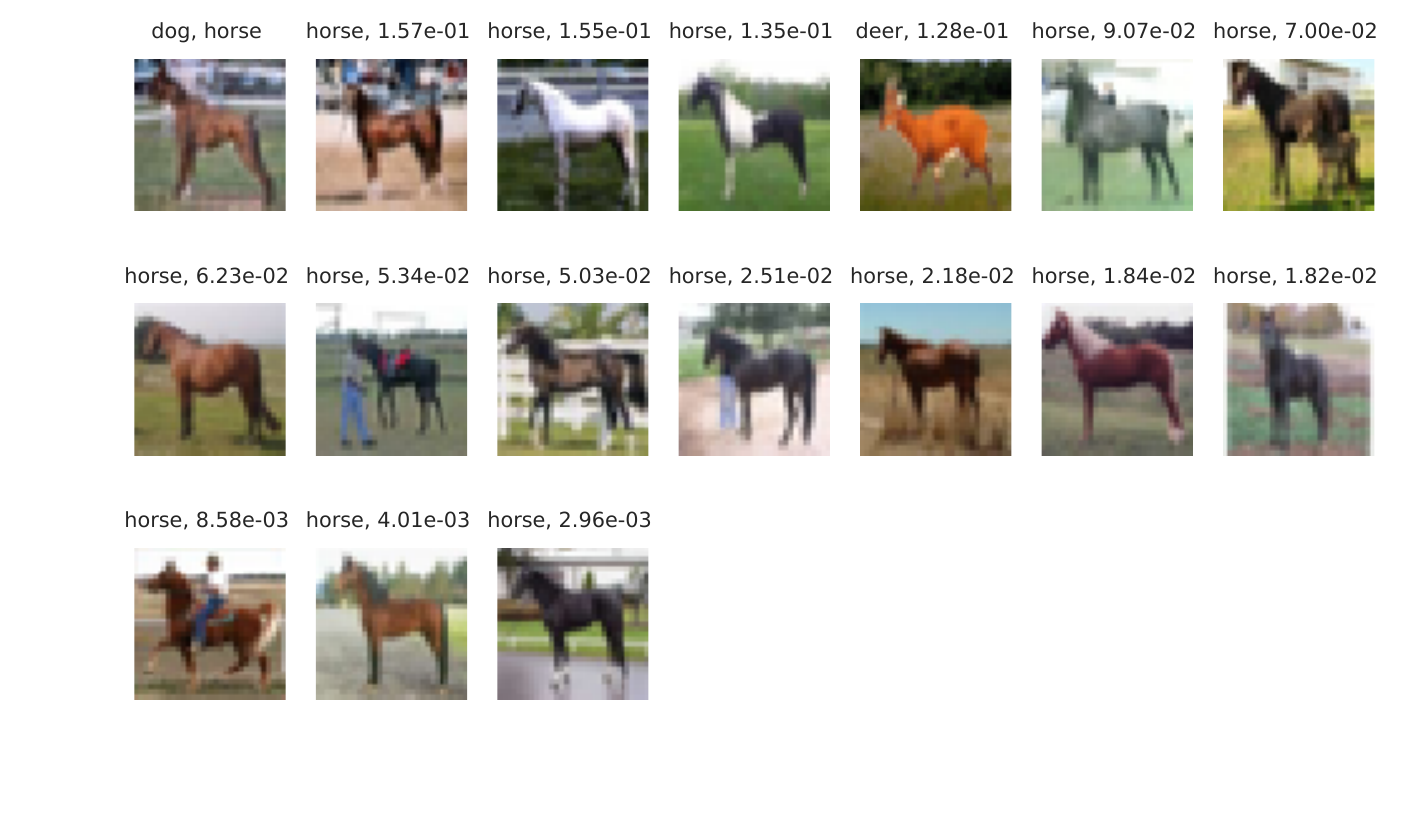}
    \end{subfigure}
    \caption{Two training set examples (first images in each set) observed to have maximum discrepancy in LOO NNK interpolation score, as well as their respective neighbors, for $k=50$. We show the assigned and predicted label for the image being classified, and assigned label and NNK weight for the neighbors. These instances exemplify the possible brittleness of the classification model, which can better inform a user about the limits of the model they are working with.}
    \label{fig:explainability_example_brittleness}
\end{figure}

Finally, we present experimental analysis of generative and adversarial images from the perspective of NNK interpolation. We study these methods using our DeepNNK framework applied on a Wide-ResNet-28-10 \cite{zagoruyko2016wide} architecture trained with autoaugment \cite{cubuk2019autoaugment}\footnote{DeepNNK achieves $97.3\%$ test accuracy on CIFAR-10 similar to that of the original network.}. 
\begin{figure}[tbp]
    \centering
    \begin{subfigure}{0.24\textwidth}
    \includegraphics[width=\textwidth]{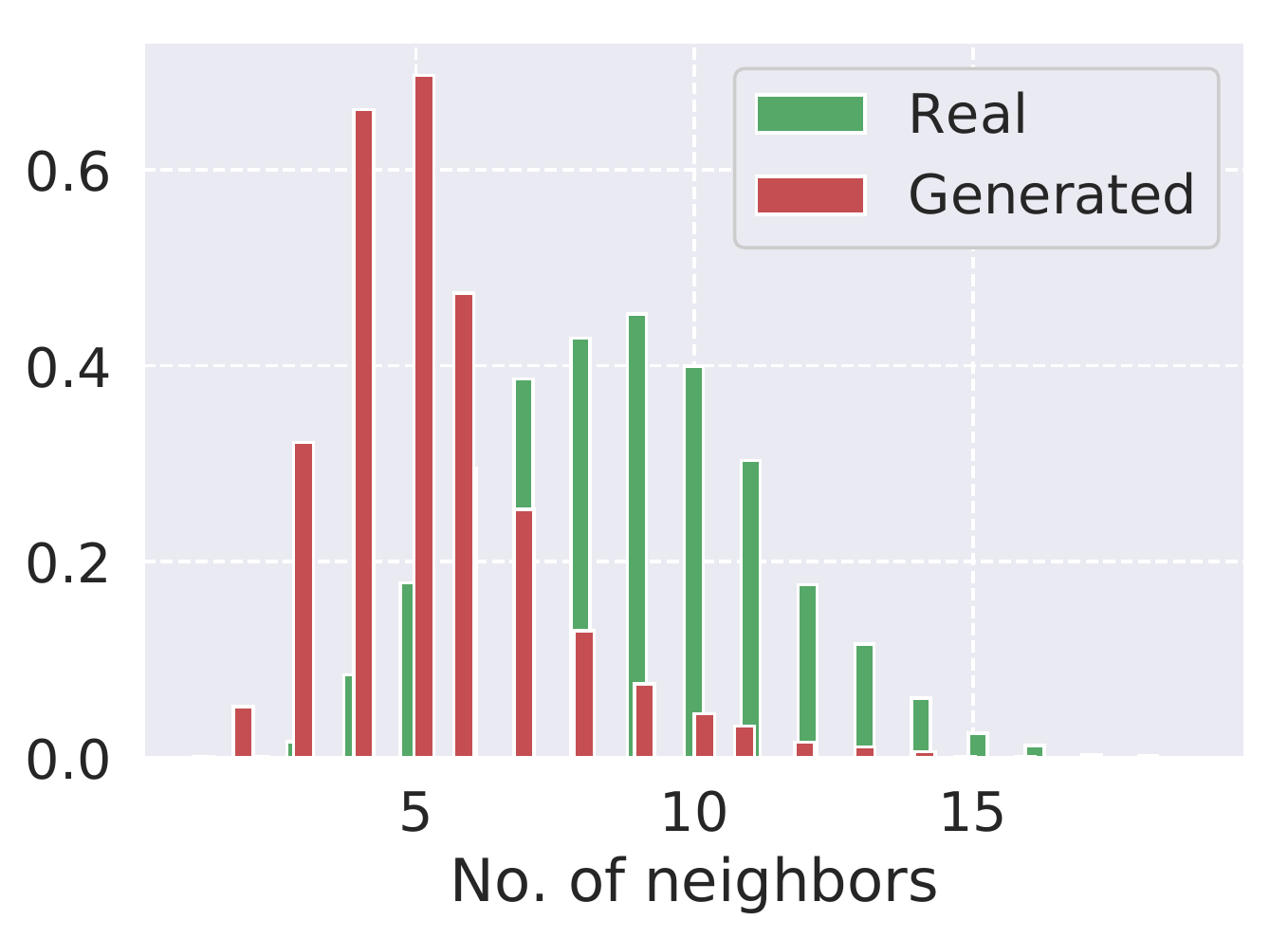}
    \caption{}
    \end{subfigure}
    \begin{subfigure}{0.24\textwidth}
    \includegraphics[width=\textwidth]{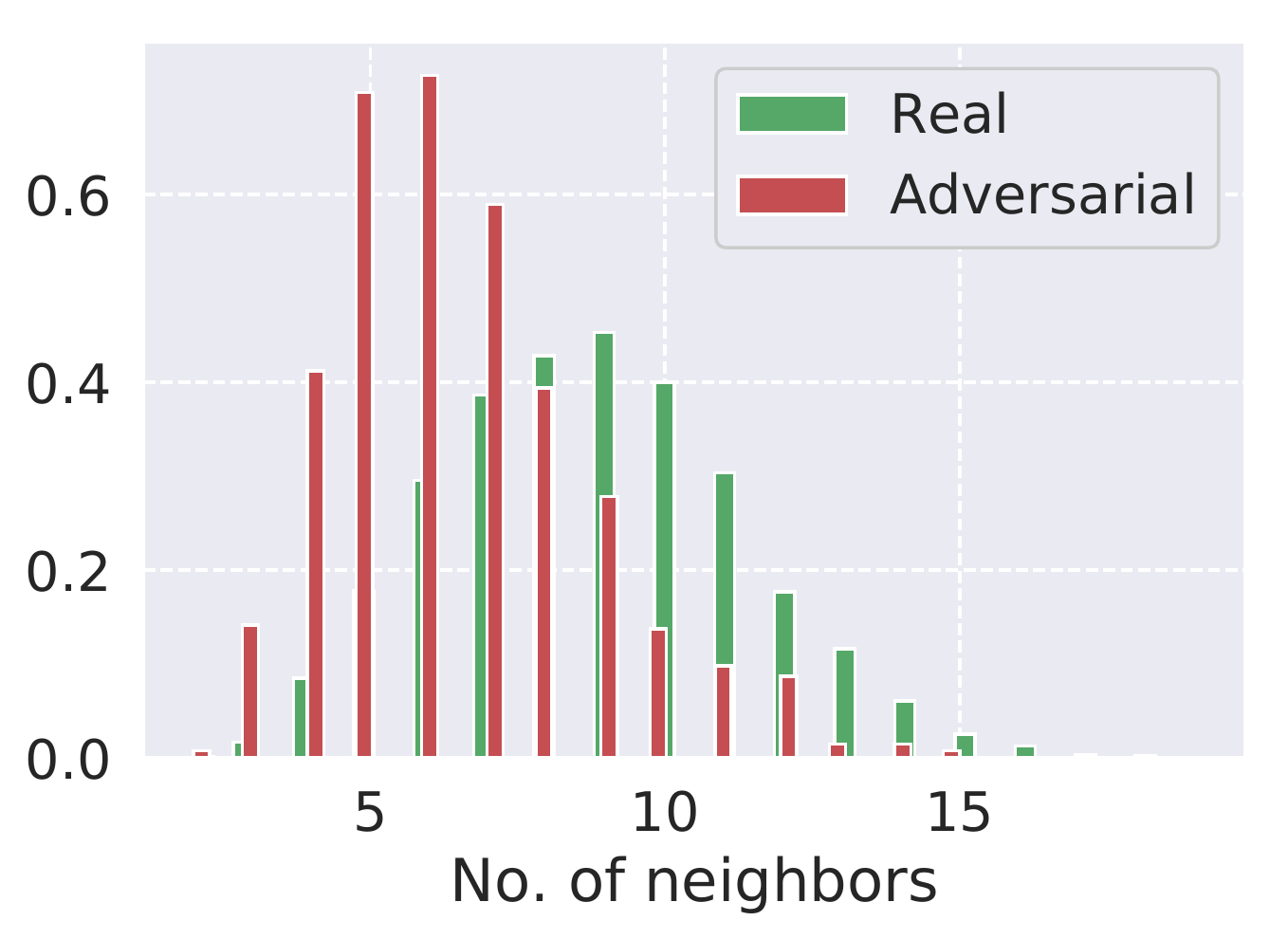}
    \caption{}
    \end{subfigure}
    \caption{Histogram (normalized) of number of neighbors for (a) generated images \cite{song2019generative}, (b) black box adversarial images \cite{croce2020reliable} and actual CIFAR-10 images. We see that generated and adversarial images on average have fewer neighbors than real images suggestive of the fact these examples often fall in interpolating regions where few train images span the space. An adversarial image is born when these areas of interpolation belong to unstable regions in the classification surface. }
    \label{fig:generative_adversarial_neighbors}
\end{figure}
\begin{figure}[tbp]
    \centering
    \begin{subfigure}{0.47\textwidth}
    \includegraphics[width=\textwidth]{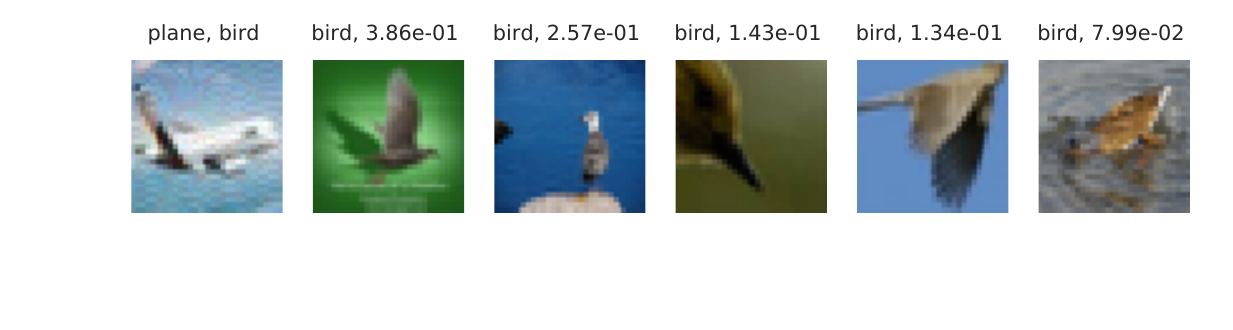}
    \end{subfigure}
    \begin{subfigure}{0.47\textwidth}
    \includegraphics[width=\textwidth]{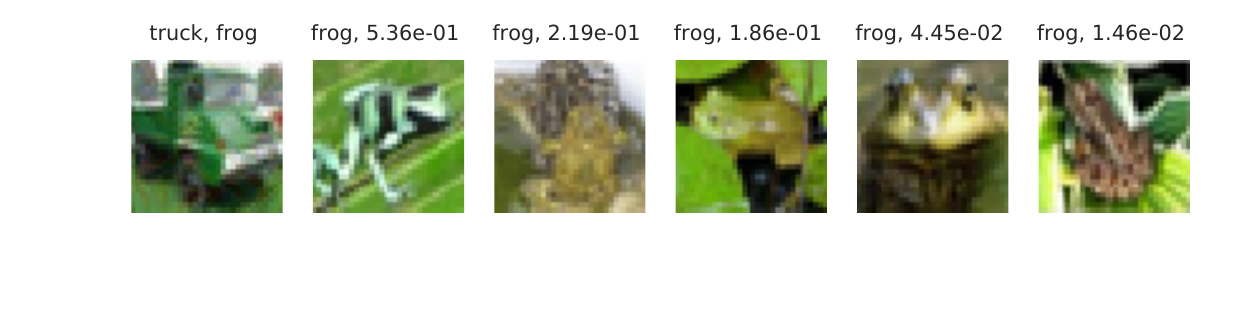}
    \end{subfigure}
    \caption{Selected black box adversarial examples (first image) and their NNK neighbors from CIFAR-10 training dataset with $k=50$. Though changes in the input image is imperceptible to a human eye, one can better characterize a prediction using NNK by observing the interpolation region of the test instance.}
    \label{fig:explainability_example_adversarial}
\end{figure}
Generative and adversarial examples leverage interpolation spaces where a model (discriminator in the case of generative images or a classifier in the case of black box attacks) is influenced by a smaller number of neighboring points. 
This is made evident in \Figref{fig:generative_adversarial_neighbors} where we see that the number of neighbors in the case of generative and adversarial images is on average smaller than that of real images. 
We conjecture that this is a property of interpolation where realistic images can be obtained in compact interpolation neighborhoods with perturbations along extrapolating, mislabeled sample directions producing adversarial images. 
Though the adversarial perturbations in the input image space is visually indistinguishable, the change in the embedding of the adversarial image in the interpolation space is significantly larger, in some cases as in \Figref{fig:explainability_example_adversarial}, belonging to regions completely different from its class. 
\section{Discussion and Future Work}
We discussed various ideas, theoretical and practical, from model interpretability and generalization to adversarial, generative examples. Underlying each of these applications is a single common tool, a local polytope interpolation, whose neighborhood support is determined automatically and is dependent on the input data. 
DeepNNK provides a way to incorporate recent theoretical works on interpolation and leads to better understanding of deep learning models by tracing their predictions back to the input data it was trained on.
We hope these attempts help progress neural networks to more real world scenarios and motivate further studies, methods of diagnosing machine models from the lens of the training data.

We conclude with few open thoughts and questions. 
\begin{itemize}
    \item 
    Leave one out is a particular instance of the more general problem of how a learning system predicts in response to perturbations of its parameters and data. We believe other kind of perturbations could help better understand neural networks, statistically as well as numerically. 
    \item 
    The error in data interpolation of \eqref{eq:nnk_kernel_objective} can be observed as that of data noise or alternatively error arising due to absence of examples in some directions (extrapolation). In either scenario, this error can be used to characterize a notion of distance between the data being interpolated and that available for interpolation. 
    We believe such a measure could help identify datasets shifts in an unsupervised manner with possible applications in domain adaptation, transfer learning.
\end{itemize}
\bibliographystyle{IEEEtran}
\bibliography{root}

\onecolumn
\section*{Supplementary Material}
\subsection{Proof of Proposition \plainref{prop:nnk_interpolation}}

\begin{lemma}[\cite{shekkizhar2019graph}]
\label{lemma:active_set_solution}
The quadratic optimization problem of (\ref{eq:nnk_kernel_objective}) satisfies active constraints set\footnote{In constrained optimization problems, some constraints will be strongly binding, i.e., the solution to optimization at these elements will be zero to satisfy the KKT condition of optimality. These constraints are referred to as \textbf{active constraints}, knowledge of which helps reduce the problem size as one can focus on the inactive subset that requires optimization. The constraints that are active at a current feasible solution will remain active in the optimal solution \cite{bernau1990active}.}. 
Given a partition $\{\vtheta_P, \vtheta_{\bar{P}}\}$, where $\vtheta_P > 0$ (inactive) and $\vtheta_{\bar{P}} = 0$ (active), the solution $[\vtheta_P \; \vtheta_{\bar{P}}]^\top$ is the optimal solution provided:
\begin{align*}
    \mK_{P,P}\vtheta_P = \mK_{P,*} \quad \text{and} \quad
    \mK_{P,{\bar{P}}}^\top\vtheta_P - \mK_{{\bar{P}},*} \geq \bf0
\end{align*}
\end{lemma}
Moreover, the set $P$ corresponds to non zero support of the constrained problem if and only if $\mK_{P, P}$ is full rank and $\vtheta_{P} > 0$ \cite{nguyen2019non}. Thus, the solution to (\ref{eq:nnk_kernel_objective}) is obtained as
\begin{align}
 \vtheta_S = [\vtheta_P \; \vtheta_{\bar{P}}]^\top = [(\mK_{P,P})^{-1}\mK_{P,*} \;\; \bf0] \label{eq:active_set_solution}
\end{align}
\begin{proof}[Proof of Proposition \ref{prop:nnk_interpolation}]
 Let $\mPhi_P$ correspond to the matrix containing the $\hat{k}$ neighbors with non zero data interpolation weights and $\vy_P$ the associated labels. The kernel space linear interpolation estimator is obtained by solving
 \begin{align}
     \tilde{\valpha} = \argmin_{\valpha} \sum_{i=1}^{\hat{k}}(y_i - \valpha^\top\vphi(\vx_i)) = \mK_{P,P}^{-1}\mPhi_P \vy_P
 \end{align}
 Therefore, using matrix identity and \eqref{eq:active_set_solution} resulting from lemma \ref{lemma:active_set_solution},  the estimate $y$ at $\vx$ is obtained as 
\begin{align*}
 y=\tilde{\valpha}^\top\vphi(\vx) = \vy_P^\top\mPhi_P^\top\mK_{P,P}^{-1}\vphi(\vx)
 = \vy_P^\top\mK_{P,P}^{-1}\mPhi_P^\top\vphi(\vx) =  \vy_P^\top\mK_{P,P}^{-1}\mK_{P,*} = \vy_P^\top\vtheta_P = \sum_{i=1}^{\hat{k}}\vtheta_{i}\;y_i
\end{align*}
\end{proof}

\subsection{Proof of Theorem \plainref{thm:excess_mean_sq_risk}}
\begin{proof}
The proof follows a similar argument as in the simplicial interpolation bound in \cite{belkin2018overfitting}.
The expected excess mean squared risk can be partitioned based on disjoint sets as\footnote{All expectation in this proof are condition on $D_{train}$. For the sake of brevity, we do not make this conditioning explicit in our statments.} 
\begin{align}
    \E_X[(\hat{\eta}(\vx) - \eta(\vx))^2] &= \E_X[(\hat{\eta}(\vx) - \eta(\vx))^2 | X \notin \gC] P(X \notin \gC) + \E_X[(\hat{\eta}(\vx) - \eta(\vx))^2 | X \in \gC] P(X \in \gC) \nonumber \\
    &\leq \E_X[(\hat{\eta}(\vx) - \eta(\vx))^2 | X \notin \gC] P(X \notin \gC) + \E_X[(\hat{\eta}(\vx) - \eta(\vx))^2 | X \in \gC] \label{eq:thm1_risk_parition}
\end{align}
For points outside the convex hull, NNK extrapolates labels and no guarantees can be made on the regression without further assumptions. Thus, $(\hat{\eta}(\vx) - \eta(\vx))^2 \leq 1$ which reduces the first term on the left of \eqref{eq:thm1_risk_parition} to that of theorem.


Let $\vtheta_{\hat{k}}$ be the solution to NNK interpolation objective  \plainref{eq:nnk_kernel_objective}. Let $w_i = \frac{\vtheta_{i}}{ \sum_{i=1}^{\hat{k}}\vtheta_{i}}$ denote the weight normalized values. The normalized weights follow a $\text{Dirichlet(1, 1 \dots 1)}$ distribution with $\hat{k}$ concentration parameters. 
\begin{align}
    \hat{\eta}(\vx) - \eta(\vx) &=  \sum_{i=1}^{\hat{k}}w_i(y_i - \eta(\vx)) 
    = \sum_{i=1}^{\hat{k}}w_i(y_i - \eta(\vx_i) + \eta(\vx_i) - \eta(\vx)) =\sum_{i=1}^{\hat{k}}w_i \epsilon_i + \sum_{i=1}^{\hat{k}}w_i b_i \label{eq:bias_error_eta_decomposition}
\end{align}
where $\epsilon_i = y_i - \eta(\vx_i)$ corresponds to Bayesian estimator errors in the training data and $b_i =  \eta(\vx_i) - \eta(\vx)$ is related to bias. By smoothness assumption on $\eta$ we have
\begin{align}
    |b_i| = |\eta(\vx_i) - \eta(\vx)| \leq A||\vphi(\vx_i) - \vphi(\vx)||^\alpha \leq A\delta^\alpha \label{eq:bias_bound}
\end{align}
Since $b_i$ and $\epsilon_i$ are independent, we have
\begin{align}
    \E_X[(\hat{\eta}(\vx) - \eta(\vx))^2 | X \in \gC] = \E_X[\left(\sum_{i=1}^{\hat{k}}w_i \epsilon_i\right)^2 | X \in \gC] + \E_X[\left(\sum_{i=1}^{\hat{k}}w_i b_i\right)^2| X \in \gC] \label{eq:bias_error_risk_decomposition}
\end{align}
By Jensen's inequality, $\left(\sum_{i=1}^{\hat{k}}w_i b_i\right)^2 \leq \sum_{i=1}^{\hat{k}}w_i b_i^2$ and bound in \eqref{eq:bias_bound},
\begin{align}
    \E_X[\left(\sum_{i=1}^{\hat{k}}w_i b_i\right)^2| X \in \gC] \leq \E_X[\sum_{i=1}^{\hat{k}}w_i b_i^2 | X \in \gC] \leq \E_X[\sum_{i=1}^{\hat{k}}w_i A^2\delta^{2\alpha} | X \in \gC] = A^2\delta^{2\alpha} \label{eq:bias_reduced_form}
\end{align}
Let $\nu(\vx) = var(Y|X=\vx)$. Under independence assumption on noise, the term with $\epsilon$ in \eqref{eq:bias_error_risk_decomposition} can be rewritten as
\begin{align*}
    \E_X[\left(\sum_{i=1}^{\hat{k}}w_i \epsilon_i\right)^2 | X \in \gC] &= \E_X[\sum_{i=1}^{\hat{k}}w_i^2 \epsilon_i^2 | X \in \gC] = \E_K\left[\sum_{i=1}^{\hat{k}}\E_{X|K}[w_i^2| X \in \gC] \E_{X|K}[\epsilon_i^2| X \in \gC]\right] 
    \\
    &\leq \E_K\left[\frac{2}{(\hat{k} + 1)(\hat{k})} \sum_{i=1}^{\hat{k}} \nu(\vx_i)\right]
    \leq \E_K\left[\frac{2}{(\hat{k} + 1)(\hat{k})} \sum_{i=1}^{\hat{k}} \nu(\vx) + |\nu(\vx_i) - \nu(\vx)|\right]
\end{align*}
where we use the fact that $w_i$ follows Dirichlet distribution. Now, the smoothness assumption on $var(Y|X)$ allows us to bound
\begin{align}
    |\nu(\vx_i) - \nu(\vx)| \leq A'||\vphi(\vx_i) - \vphi(\vx)||^{\alpha'} \leq A'\delta^{\alpha'}\\
    \implies \E_X[\left(\sum_{i=1}^{\hat{k}}w_i \epsilon_i\right)^2 | X \in \gC] \leq \frac{2}{(\E_K[\hat{k}] + 1)} \left(\nu(\vx) + A'\delta^{\alpha'}\right) \label{eq:var_reduced_form}
\end{align}
Combining with \eqref{eq:bias_reduced_form}, the risk bound for points within the convex hull of training data is obtained as
\begin{align}
    \E_X[(\hat{\eta}(\vx) - \eta(\vx))^2 | X \in \gC] \leq  A^2\delta^{2\alpha} +  \frac{2}{(\E_K[\hat{k}] + 1)} \left(\nu(\vx) + A'\delta^{\alpha'}\right) \label{eq:second_term_thm1}
\end{align}
\Eqref{eq:second_term_thm1} along with the reduction for points outside the convex hull $\gC$ obtained earlier gives the excess risk bound and concludes the proof.
\end{proof}

\subsection{Proof of Corollary \plainref{coroll:excess_mean_sq_risk_convergence}}
\begin{proof}
The nearest neighbor convergence lemma of \cite{cover1967nearest} states that for an i.i.d sequence of random variables $\gD = \{\vx_1, \vx_2 \dots \vx_N\}$ in $\R^d$, the nearest neighbor of $\vx$ from the set $\gD$ converges in probability, $NN(\vx) \rightarrow_p \vx$. Equivalently, this would correspond to convergence in kernel representation of the data points. Thus, the solution to NNK data interpolation objective is reduced to $1$-nearest neighbor interpolation with $\E_K[\hat{k}] = 1$ and $\limsup_{N\rightarrow \infty}\delta = 0$.
Now, under the assumption that the $supp(\mu)$ belongs to a polytope, the first term on the right of \eqref{eq:excess_sq_risk} vanishes i.e.,
$\limsup_{N\rightarrow\infty}E_X[\mu(\R^d \backslash \gC)] = 0$
\end{proof}
\subsection{Proof of Corollary \plainref{coroll:classifier_risk_convergence}}
\begin{proof}
The excess classification risk associated with this classifier is related the regression risk as 
\begin{align}
    \E_X[\gR(\hat{f}(\vx)) - \gR(f(\vx))] \leq \E_X[\sI(\hat{f}(\vx) \neq f(\vx))] \leq 2 \E_X[|\hat{\eta}(\vx) - \eta(\vx)|] \label{eq:classifier_risk_sq_risk_relation}
\end{align}
From Corollary \ref{coroll:excess_mean_sq_risk_convergence}, we have
\begin{align*}
    \limsup_{N\rightarrow\infty} \E_X[(\hat{\eta}(\vx) - \eta(\vx))^2] \leq \E_X[(Y - \eta(\vx)^2]
\end{align*}
By Jensen's inequality
\begin{align}
   \limsup_{N\rightarrow\infty} \left(\E_X[|\hat{\eta}(\vx) - \eta(\vx)|]\right)^2 \leq \limsup_{N\rightarrow\infty} \E_X[(\hat{\eta}(\vx) - \eta(\vx))^2] \label{eq:risk_jensen_inequality}
\end{align}
Combining with \eqref{eq:classifier_risk_sq_risk_relation} gives the required risk bound.
\end{proof}
\subsection{Proof of Theorem \plainref{thm:loo_bound_theorem}}
\begin{proof}
The proof is based on the $k$-nearest neighbor result from Theorem 1 in \cite{devroye1979deleted} which states that 
\begin{align}
   P(|\gR_{loo}(\hat{\eta}|\gD_{train}) - \gR_{gen}(\hat{\eta})| > \epsilon) \leq  2e^{-N\epsilon^2/18} +6e^{-N\epsilon^3/\left(108k(2 + \gamma)\right)}
\end{align}
As in \cite{devroye1979deleted}, where the result is extended based on the $1$-nearest neighbor, here it suffices to replace $k$ by $\E_K[\hat{k}]$ since each data point on average cannot be NNK neighbors to more than $\E_K[\hat{k}]\gamma + 2 \leq \E_K[\hat{k}](\gamma + 2)$ data points.
\end{proof}
\end{document}

%% file: math_commands.tex
\usepackage{amsmath,amsfonts, amsthm, bm, amssymb}

\def\Figref#1{Figure~\ref{#1}}
\def\twoFigref#1#2{Figures \ref{#1} and \ref{#2}}

\def\eqref#1{equation~(\ref{#1})}
\def\Eqref#1{Equation~(\ref{#1})}

\def\plainref#1{(\ref{#1})}

\def\1{\bm{1}}

\def\vtheta{{\bm{\theta}}}
\def\valpha{{\bm{\alpha}}}

\def\vphi{{\bm{\phi}}}

\def\vh{{\bm{h}}}

\def\vw{{\bm{w}}}
\def\vx{{\bm{x}}}
\def\vy{{\bm{y}}}

\def\mK{{\bm{K}}}

\def\mPhi{{\bm{\Phi}}}

\DeclareMathAlphabet{\mathsfit}{\encodingdefault}{\sfdefault}{m}{sl}
\SetMathAlphabet{\mathsfit}{bold}{\encodingdefault}{\sfdefault}{bx}{n}

\def\gC{{\mathcal{C}}}
\def\gD{{\mathcal{D}}}

\def\gH{{\mathcal{H}}}

\def\gK{{\mathcal{K}}}
\def\gL{{\mathcal{L}}}

\def\gR{{\mathcal{R}}}

\def\sI{{\mathbb{I}}}

\newcommand{\E}{\mathbb{E}}

\newcommand{\R}{\mathbb{R}}

\DeclareMathOperator*{\argmin}{arg\,min}

\def\<#1,#2>{\langle #1,\,#2\rangle}